\documentclass[acmsmall,nonacm]{acmart} %

\acmConference[PODS '26]{ACM SIGMOD/PODS Conference 2026}{May 31–June 5, 2026}{Bengaluru, India}

\usepackage[T1]{fontenc}
\usepackage[utf8]{inputenc}
\usepackage{microtype}
\usepackage{mathtools, amsmath, amssymb, bm}
\usepackage{makecell}
\usepackage{thmtools, thm-restate}
\usepackage{algorithm, algpseudocode}
\usepackage{booktabs}
\usepackage{nicefrac}
\usepackage{csquotes}        %
\usepackage[dvipsnames]{xcolor} 
\usepackage{xspace}
\let\vec\mathbf
\makeatletter
\newcommand{\DiamondPlus}{\mathbin{\mathpalette\DiamondPlus@aux\relax}}
\newcommand{\DiamondPlus@aux}[2]{%
  \ooalign{%
    $\m@th#1\Diamond$\cr
    \hidewidth$\m@th#1\raisebox{0.0ex}{\scalebox{1}[1.1]{$+$}}$\hidewidth\cr
  }%
}
\makeatother
\newcommand{\mpDiamond}{\DiamondPlus}
\newcommand{\mpdDiamond}{\!\DiamondPlus\!}
\newcommand{\mlDiamond}{\Diamond}
\usepackage{colortbl} 

 \usepackage{comment,todonotes} 
\theoremstyle{plain}
\declaretheorem[name=Theorem,numberwithin=section]{theorem}
\newtheorem{lemma}[theorem]{Lemma}
\newtheorem{proposition}[theorem]{Proposition}
\newtheorem{corollary}[theorem]{Corollary}

\theoremstyle{definition}
\newtheorem{definition}[theorem]{Definition}
\newtheorem{example}[theorem]{Example}

\newcommand{\klara}[1]{\todo[color=red!25,inline]{Klara: #1}}
\newcommand{\floris}[1]{\todo[color=olive!25,inline]{Floris: #1}}
\newcommand{\jan}[1]{\todo[color=teal!25,inline]{Jan: #1}}

\newcommand{\matlang}{\textsf{MATLANG}\xspace}
\newcommand{\tl}{\textsf{TL}\xspace}
\newcommand{\mplang}{\textsf{MPLang}\xspace}

\newcommand{\amplang}[0]{\textsf{A-MPLang}\xspace}
\newcommand{\smplang}[1]{#1\textsf{-}\mplang}
\newcommand{\relu}[0]{\textsf{ReLU}\xspace}
\newcommand{\relub}[0]{\textsf{ReLU}}

\newcommand{\sign}[0]{\textsf{sign}\xspace}
\newcommand{\trrelu}[0]{\textsf{TrReLU}\xspace}

\newcommand{\step}[0]{\textsf{bool}\xspace}
\newcommand{\stepb}[0]{\textsf{bool}}
\newcommand{\bool}[0]{\textsf{bool}\xspace}

\newcommand{\N}{\mathbb{N}}
\newcommand{\Z}{\mathbb{Z}}
\newcommand{\Q}{\mathbb{Q}}
\newcommand{\R}{\mathbb{R}}

\newcommand{\Scal}{\mathcal S}

\graphicspath{ {./images/} }
\theoremstyle{remark}
\newtheorem{remark}[theorem]{Remark}

\usepackage[nameinlink,noabbrev]{cleveref}

\begin{document}

\title[A Logical View of GNN-Style Computation and the Role of Activation Functions]{\texorpdfstring{A Logical View of GNN-Style Computation and the Role of Activation Functions}{A Logical View of GNN-Style Computation and the Role of Activation Functions}
}

\author{Pablo Barcel\'o}
\orcid{0000-0003-2293-2653}
\affiliation{%
  \institution{Institute for Mathematical and Computational Engineering, \\ Pontifical Catholic University of Chile \& IMFD \& CENIA}
  \city{Santiago}
  \country{Chile}
}
\email{pbarcelo@uc.cl}

\author{Floris Geerts}
\orcid{0000-0002-8967-2473}
\affiliation{%
  \institution{Department of Computer Science, University of Antwerp}
  \city{Antwerp}
  \country{Belgium}
}
\email{floris.geerts@uantwerp.be}

\author{Matthias Lanzinger}
\orcid{0000-0002-7601-3727}
\affiliation{%
  \institution{Institute for Logic and Computation, TU Wien}
  \city{Vienna}
  \country{Austria}
}
\email{matthias.lanzinger@tuwien.ac.at}

\author{Klara Pakhomenko}
\orcid{0009-0009-9715-2214}
\affiliation{%
  \institution{Data Science Institute, Universiteit Hasselt}
  \city{Diepenbeek}
  \country{Belgium}
}
\email{klara.pakhomenko@uhasselt.be}

\author{Jan Van den Bussche}
\orcid{0000-0003-0072-3252}
\affiliation{%
 \institution{Data Science Institute, Universiteit Hasselt}
  \city{Diepenbeek}
  \country{Belgium}
}
\email{jan.vandenbussche@uhasselt.be}

\begin{abstract}
We study the numerical and Boolean expressiveness of \mplang, a declarative language that captures the computation of graph neural networks (GNNs) through linear message passing and activation functions. We begin with \amplang, the fragment without activation functions, and give a characterization of its expressive power in terms of walk-summed features.
For bounded activation functions, we show that (under mild conditions) all eventually constant activations yield the same expressive power—numerical and Boolean—and
that it subsumes previously established logics for GNNs with eventually
constant activation functions but without linear layers.
Finally, we prove the first expressive separation between unbounded and bounded activations in the presence of linear layers:
\mplang with \relu is strictly more powerful for numerical queries than \mplang with eventually constant activation functions, e.g., 
truncated \relu.
This hinges on subtle interactions between linear aggregation and eventually constant non-linearities, and it establishes that GNNs using \relu\ 
are more expressive than those restricted to eventually constant activations
and linear layers.
\end{abstract}

\keywords{Liouville number, symmetric function, simple function, graded modal logic}

\ccsdesc[500]{Computing methodologies~Machine learning}
\ccsdesc[500]{Theory of computation~Database query languages (principles)}
\ccsdesc[500]{Mathematics of computing~Graph coloring}

\maketitle
\section{Introduction}
\paragraph{\bf Context.} 
The study of query languages for graph-structured data has traditionally focused on structural and navigational queries, where the goal is to select tuples of nodes and/or paths based on the topology of the underlying graph \cite{DBLP:conf/pods/Baeza13}. This line of work has produced a rich theory of languages such as \emph{regular path queries} \cite{DBLP:conf/sigmod/CruzMW87,DBLP:journals/siamcomp/MendelzonW95}, \emph{conjunctive} regular path queries \cite{DBLP:conf/kr/CalvaneseGLV00}, and their extensions \cite{DBLP:conf/pods/ConsensM90,DBLP:journals/tods/BarceloLLW12,DBLP:journals/jacm/LibkinMV16,DBLP:journals/mst/ReutterRV17}, together with extensive results on expressive power, complexity of query evaluation, and static analysis \cite{DBLP:journals/jcss/FreydenbergerS13,DBLP:conf/lics/0001BV17,DBLP:conf/icalp/BarceloF019,DBLP:conf/kr/FigueiraGKMNT20,DBLP:conf/pods/FigueiraR22,DBLP:conf/icdt/Figueira20,DBLP:conf/icdt/FeierGM24,DBLP:journals/lmcs/FigueiraM25,DBLP:journals/pacmmod/FigueiraMR25}. These foundational developments have also influenced language design in practice, informing standards and systems such as SPARQL \cite{SPARQL11}, G-CORE \cite{DBLP:conf/sigmod/AnglesABBFGLPPS18}, Cypher \cite{DBLP:conf/sigmod/FrancisGGLLMPRS18}, and, most recently, the ISO-standardized GQL \cite{DBLP:conf/icdt/FrancisGGLMMMPR23}. 

A complementary perspective on querying graph data has emerged from the success of {\em graph neural networks} (GNNs) \cite{4700287,Gil+2017,Hamilton+2017,KipfWelling2017}. Rather than selecting nodes or paths, GNNs compute numerical node embeddings via local {\em message passing}, combining linear aggregation with a (typically non-linear) activation function such as \relu, truncated \relu, \sign, or \step. Each GNN therefore defines a {\em numerical} node query and, after thresholding, also a {\em Boolean} one.
This computational pattern is captured by \mplang \cite{DBLP:conf/foiks/GeertsSB22},
a language influenced by modal logics for semiring-annotated databases
\cite{graedel-prov-modal}.
It is a declarative language rooted in linear-algebraic frameworks such as \matlang \cite{brijder2019expressive} and the tensor language \tl for analysing GNN expressiveness \cite{geerts2022expressiveness}. \mplang expresses GNNs through linear combination, neighbourhood aggregation, and an activation function $\sigma$, providing a clean, logic-oriented setting for reasoning about embedding-based graph queries. This aligns with recent efforts to characterize the expressive power of GNN architectures \cite{DBLP:conf/iclr/BarceloKM0RS20,DBLP:journals/theoretics/Grohe24,DBLP:conf/icalp/BenediktLMT24,DBLP:conf/nips/AhvonenHKL24,nunn2024logic,DBLP:journals/corr/abs-2505-11930,DBLP:journals/corr/abs-2508-06091,DBLP:journals/corr/abs-2506-13911}.

\paragraph{\bf Problems studied.}
Our goal is to characterize the expressive power of \mplang-style languages in terms of their {\em numerical} and {\em Boolean} queries (the former assign real values to nodes, the latter values in~$\{0,1\}$). As a baseline for understanding the role of activation functions, we start with \amplang, the affine fragment of \mplang \emph{without} activations. We analyse the queries definable in \amplang via the walks originating at the node under evaluation and compare their Boolean expressive power with that of modal logic (ML), a standard benchmark for GNN expressiveness \cite{DBLP:conf/iclr/BarceloKM0RS20,DBLP:conf/icalp/BenediktLMT24,horrocks}. We also study closure properties of \amplang under Boolean combinations.

We then consider the general case of \smplang{$\sigma$}, which extends \amplang with expressions of the form $\sigma(e)$ applying an activation function~$\sigma$ to the embedding computed by~$e$. Following standard distinctions, we treat {\em bounded} and {\em unbounded} activation functions separately \cite{DBLP:journals/theoretics/Grohe24,DBLP:conf/icalp/BenediktLMT24}. For bounded activations, we focus on \emph{eventually constant} functions \cite{DBLP:conf/icalp/BenediktLMT24}, a natural analogue of piecewise-linear and threshold operations such as truncated {\sf ReLU}, {\sf sign}, and {\sf bool}. For such~$\sigma$, we analyse \smplang{$\sigma$} again in terms of walks and examine whether different eventually constant~$\sigma$ differ in expressive power. For unbounded activations, such as \relu, we ask whether they yield strictly more expressive power than any eventually constant function—specifically, whether there exist queries definable in \smplang{\relub} that no \smplang{$\sigma$} with eventually constant~$\sigma$ can express.

\paragraph{\bf Our results.} We consider {\em embedded} graphs, where each node has an associated feature vector. A central case is that of {\em coloured} graphs, whose embeddings are one-hot encodings of finitely many colours. Most results are stated for coloured graphs, with a few exceptions that extend to general embeddings. Crucially, all inexpressibility results already hold in the coloured setting.

\begin{itemize}

\item \underline{No activation function.}
We begin by analysing \amplang, the fragment of \mplang without activation functions. 
Numerical queries definable in \amplang correspond to
affine combinations of
walk-counts and walk-summed features
(Theorem~\ref{amplang-normalform}).
Using this, we analyse how far
the Boolean queries expressible by \amplang go as a Boolean query
language and show that it is
not closed under Boolean combinations,
making it { incomparable} to Modal Logic (ML)
on coloured graphs (\Cref{prop:nonclosure}).
Since \amplang induces exactly the embeddings computed by GNNs with linear activations, this confirms that non-linearities are essential for capturing several meaningful graph properties, and in particular, that the ability of GNNs to match the expressive power of ML \cite{DBLP:conf/iclr/BarceloKM0RS20} crucially depends on them.

\item \underline{Bounded activation functions.}
We show that, under mild assumptions, any eventually constant activation function $\sigma$ yields the same expressive power for numerical queries in \smplang{$\sigma$} over coloured graphs (Theorem~\ref{thm:eventually-constant-activation-functions}),
improving a related result in \cite{DBLP:conf/icalp/BenediktLMT24}
on Boolean queries computed by GNNs without linear layers.
This shows that \smplang{$\sigma$} with any such eventually constant $\sigma$
is closed under Boolean operations. 
Our languages subsume the logics introduced by Benedikt et al.~\cite{DBLP:conf/icalp/BenediktLMT24} and Nunn et al.~\cite{nunn2024logic} for Boolean GNN expressivity with eventually constant activations but without linear layers.

\item \underline{Beyond bounded activation functions.} 
We then turn to unbounded non-linear activation functions and compare their
expressive power with that of arbitrary eventually constant ones. Our main
technical result is that, for numerical queries over coloured graphs, \smplang{\relub} is strictly
more expressive than \mplang with any eventually constant activation
(Theorem~\ref{thm:relu-vs-trrelu}). This is the most technically demanding part
of the paper and shows that GNNs using both {\sf ReLU} and linear layers surpass
those relying only on truncated {\sf ReLU} with linear layers. A related
separation was shown by Benedikt et al.~\cite{DBLP:conf/icalp/BenediktLMT24},
but only for GNNs without linear layers. Incorporating identity activations is
substantially more challenging, as linear layers interact delicately with
non-linear transformations and isolating their contribution is essential
\cite{sammy,DBLP:conf/nips/BarlagHSVV24}. Overall, unbounded activations like
{\sf ReLU} permit numerical computations that bounded ones cannot reproduce,
leading to different expressive capabilities.

 \end{itemize}

\paragraph{{\bf Types of embeddings and queries.}}
{
Previous work on uniform logical expressivity of GNNs has largely focused on coloured graphs (Boolean node features) and Boolean queries \cite{DBLP:conf/lics/Grohe21, DBLP:conf/iclr/BarceloKM0RS20, DBLP:conf/icalp/BenediktLMT24}. We follow this tradition, but also study numerical queries and, where possible, extend results to general $d$-embeddings. As a guiding principle, our separation/inexpressibility results are proved already in the most restrictive setting (Boolean queries on colourings), while our containment results are stated in the most general setting we can obtain (often numerical queries over numerical embeddings). In the Conclusion we overview our results in \Cref{tbl:summary}. Finally, although applied GNNs typically start from real-valued attributes, discretization provides a direct link to our setting: continuous features can be quantized into finitely many bins and encoded as colours. This means our negative results remain informative under richer inputs, whereas extending some of our positive collapse/equivalence results beyond colourings is an open question.
}

\paragraph{{\bf Our techniques.}}
{
Prior work has analysed GNNs through their distinguishing power
\cite{DBLP:conf/lics/Grohe21}, and also in non-uniform variants where the
model (or an additional readout) is allowed to exploit the size of the input
graph \cite{DBLP:conf/iclr/BarceloKM0RS20}. In contrast, we study \emph{uniform}
expressivity: a single expression/network architecture must work across all
graph sizes. This perspective calls for different techniques.}
The analysis of the expressive power of the numerical query languages studied in this paper requires techniques that are less commonly encountered in classical database theory, which typically deals with query languages
based on Boolean logic. Some of our results rely on traditional methods, e.g., establishing normal forms for expressions. Likewise, inexpressibility
of a query can still be shown by exhibiting pairs of indistinguishable graphs on which the query yields different outcomes. However, the indistinguishability requirement becomes stronger, as we now need to show that all expressions return the same numerical value (or, for Boolean queries, values with the same sign). 
To this end, our proofs draw on notions from analysis, including properties of the rationals (as opposed to real numbers), alternating Liouville constructions and rational approximation of reals,
limits, symmetric functions, and both simple and piecewise polynomial functions.

\paragraph{\bf Organization of the paper.}
We present basic notions in Section \ref{sec:prelims}, introduce \mplang and its connection to GNNs
in Section \ref{sec:mplang}. Then, in Section \ref{sec:amplang} we cover \amplang, and in Sections \ref{sec:bounded} and \ref{sec:relu} we address bounded and unbounded activations, respectively. We conclude in Section \ref{sec:final}  with remarks and open problems.
Full proofs that are omitted in the main body can be found in the appendix.

\section{Preliminaries} \label{sec:prelims}

\noindent
\paragraph{Basics.}
For $k \in \mathbb{N}$, we set
$[i,j] \coloneqq \{i,i+1,\ldots,j\}$
and $[k] \coloneqq \{1,\dots,k\}$.
A function $a:\R^k \to \R$
is \textit{affine} if 
$a(x_1,\ldots,x_k) \coloneqq a_0 + \sum_{i=1}^k a_ix_i$
for $a_0,\ldots,a_k \in \R$.
We define $\relu : \mathbb{R} \to \mathbb{R}$ by $\relu(x) = \max{\{0, x\}}$, and the {\em truncated} ReLU $\trrelu : \mathbb{R} \to \mathbb{R}$ by $\trrelu(x) = \max{\{0, \min{\{1, x\}}\}}$. 

\noindent
\paragraph{Graphs and embeddings.} 
A graph $G$ is a pair $(V,E)$, where $V$ is a finite set of nodes and 
$E \subseteq {\bigl\{ \{u,v\} \mid u,v \in V,\ u \neq v \bigr\}}$ is a set of undirected edges.
A \emph{walk} of length $n \ge 0$ in  
$G$ is a sequence of nodes 
$v_0, \dots, v_n$ such that ${\{v_k, v_{k+1}\}} \in E$, for all $k \in [0,n-1]$.  
A walk of length $0$ from a vertex $v$ is the trivial walk $(v)$.

For $d > 0$, a \emph{$d$-embedding} of a graph $G = (V,E)$ is a map $\gamma : V \to \R^d$. The pair $(G,\gamma)$ is called a {\em $d$-embedded} graph, and if $v$ is a node in $V$, then the tuple $(G,\gamma,v)$ is a {\em pointed} $d$-embedded graph.

\noindent
\paragraph{Numerical queries.} 
Our goal is to analyse the expressive power of query languages that compute and transform graph embeddings. A {\em numerical query} of type $d \to r$ maps a $d$-embedded graph $(G,\gamma)$ to an $r$-embedding $\gamma'$ of $G$. Since any such query can be decomposed into $r$ numerical queries of type $d \to 1$, we restrict attention to queries with 1-dimensional output.

\noindent
\paragraph{Boolean Queries.}
Following standard practice in the literature on the expressive power of graph 
embeddings computed by GNNs \cite{DBLP:conf/icalp/BenediktLMT24, DBLP:conf/iclr/BarceloKM0RS20}, we associate with every numerical query $Q$ of type 
$d \to 1$, for $d>0$, a corresponding \emph{Boolean} query~$Q_{\mathbb{B}}$ 
defined by
\[
Q_{\mathbb{B}}(G,\gamma)(v) \coloneqq 
\begin{cases}
1, & \text{if } Q(G,\gamma)(v) > 0,\\[2pt]
0, & \text{otherwise.}
\end{cases}
\]
Thus, a Boolean query is a numerical 
query of type $d \to 1$ whose output embedding 
assigns each node a value in $\{0,1\}$. This provides a  way to 
compare numerical query languages that compute and transform graph 
embeddings with purely Boolean ones (e.g., FO or ML, introduced below).

\paragraph{Equivalence of queries.} 
Queries $Q$ and $Q'$ are {\em numerically equivalent} on a class $C$ of embedded graphs if $Q(G,\gamma)=Q'(G,\gamma)$ for every $(G,\gamma)\in C$. Similarly, we say that $Q$ and $Q'$ are {\em Boolean equivalent} on $C$, if $Q_{\mathbb{B}}(G,\gamma)=Q'_{\mathbb{B}}(G,\gamma)$ for every $(G,\gamma)\in C$.

\paragraph{Comparisons of query languages.}
We compare query languages by containment of the numerical or Boolean queries they define on a class $C$ of embedded graphs. Thus two languages may be \emph{numerically equivalent} over $C$ (when they express the same numerical queries on $C$), or one may be \emph{Boolean strictly contained} in the other (when all its Boolean queries are expressible in the latter but not vice versa).

\section{The language \mplang and GNNs}
\label{sec:mplang}

\subsection{The language \smplang{$\Sigma$}} \label{subsec:mplang}
For any finite set $\Sigma$ of functions $\sigma: \R \to \R$, called \emph{activation functions},
we define the query language \textit{\smplang{$\Sigma$}}
over $d$-embeddings, for $d > 0$, by means of the following grammar: 
$$
e\ ::= \ 1 \, \mid \, P_i \, \mid \, ae \, \mid \, e + e \, \mid \, \mpDiamond e
\mid \, \sigma(e), \ \ \ \ \ \ \ \ \ \text{where $i \in [d]$,
$\sigma \in \Sigma$, and $a \in \R$.}
$$
Real constants $a$ in expressions are often referred to as coefficients.
{Moreover, $P_i$ can be understood as projecting on the $i$-th component.}

Each \smplang{$\sigma$} expression $e$ over $d$-embeddings defines a numerical 
query of type $d \to 1$, which is inductively defined as follows on a $d$-embedded graph $(G,\gamma)$:
\begin{itemize} 
\item If $e = 1$, then $e(G,\gamma)(v) \coloneqq  1$. 
\item If $e = P_i$, then $e(G,\gamma)(v) \coloneqq  \gamma(v)_i$, the $i$-th component of $\gamma(v)$.
\item If $e = a e_1$, then $e(G,\gamma)(v) \coloneqq  a \cdot {e_1}(G,\gamma)(v)$.
\item If $e = e_1 + e_2$, then 
$e(G,\gamma)(v) \coloneqq  {e_1}(G,\gamma)(v) + {e_2}(G,\gamma)(v)$.
\item If $e = \mpDiamond e_1$,
then $e(G,\gamma)(v)\coloneqq  \sum_{{\{u,v\} \in E}} {e_1}(G,\gamma)(u)$. 
\item If $e = \sigma(e_1)$, then $e(G,\gamma)(v) \coloneqq  \sigma\bigl({e_1}(G,\gamma)(v)\bigr)$.
\end{itemize}
{
For convenience, we write $e(G,\gamma,v) \coloneqq e(G,\gamma)(v)$ for pointed embedded graphs.
}
We use $\mpDiamond^i$ to denote $i$ consecutive $\mpDiamond$ symbols.
If $\Sigma$ contains only one function $\sigma$,
we write $\smplang{\sigma}$ instead of $\smplang{\{\sigma\}}$.
A $\smplang{\Sigma}$-expression is {\em rational} if all its coefficients are rational numbers; the set of such expressions is called {\em rational} $\smplang{\Sigma}$.

\subsection{Connections between \mplang and GNNs}

We now examine the connection between \mplang and graph neural networks (GNNs). 

\paragraph{GNNs.}
A {\em GNN layer} of type $d \to r$ is a tuple $L = (W_1, W_2, b, \sigma)$, where $W_1, W_2 \in \R^{r \times d}$, $b \in \R^r$, and $\sigma : \R \to \R$ is an activation function. Such a layer defines a numerical 
query of type $d \to r$ by mapping each $d$-embedded graph $(G,\gamma)$ to: 
\[
L(G,\gamma)(v)
 \coloneqq  
\sigma \biggl(
  W_1 \cdot \gamma(v)
   + 
  W_2 \cdot \sum_{{\{u,v\} \in E}} \gamma(u)
   + 
  b
\biggr),
\qquad\text{for each } v \in V.
\]
We abuse notation by writing $\sigma(x_1,\dots,x_r)$, for $(x_1,\dots,x_r)\in\R^r$, to denote $(\sigma(x_1),\dots,\sigma(x_r))$.

A {\em GNN} over $d$-embeddings is a sequence $\Scal = L_1,\ldots,L_n$ of layers, where $L_1$ has type $d \to s_1$, each $L_i$ (for $i \in [2,n-1]$) has type $s_{i-1} \to s_i$, and $L_n$ has type $s_{n-1} \to 1$ (with $s_1=1$ if $n=1$). This GNN defines a numerical query of type $d \to 1$ on $(G,\gamma)$ by iteratively applying the layers, i.e., 
\[
 \Scal(G,\gamma)  \ =  \  (L_n \circ \cdots \circ L_1)(G,\gamma).
\]

Let $\Sigma$ be a finite set of functions $\sigma : \R \to \R$. We denote by $\Sigma$-GNNs the class of GNNs $ \Scal$, such that each GNN layer in $ \Scal$ is of the form $(W_1,W_2,b,\sigma)$ for some $\sigma \in \Sigma$. 
If $\Sigma = \{\sigma\}$, then we simply write $\sigma$-GNNs instead of $\{\sigma\}$-GNNs. 

\paragraph{The connection}
As implicit in \cite{geerts2022expressiveness}, there is a clean correspondence between queries expressible in $\smplang{\sigma}$ and those definable by $\sigma$-GNNs augmented with identity (\textsf{id}) layers.

\begin{proposition}{\em \cite{geerts2022expressiveness}} 
\label{prop:gnn-is-mplang}
    Let $\Sigma$ be a finite set of functions
    $\sigma : \mathbb{R} \to \mathbb{R}$ and $d>0$. Then $\smplang{\Sigma}$ and $(\Sigma \cup \{{\sf id}\})$-GNNs, both over $d$-embeddings, are numerically 
    equivalent on all $d$-embedded graphs.
\end{proposition}

\begin{remark}
    Further inspecting the proof in \cite{geerts2022expressiveness}, one obtains an additional correspondence: if $\sigma$ is idempotent and the identity function {\sf id} can be expressed from $\sigma$ using affine transformations, then the queries definable by $\sigma$-GNNs over $d$-embeddings coincide with those definable in $\smplang{\sigma}$ over $d$-embeddings. Notably, \relu satisfies this property: ${\sf id}(x) = {\sf ReLU}(x) - {\sf ReLU}(-x)$. \qed
\end{remark}

\subsection{Traditional query languages on graphs.}

Classical query languages for graphs operate on node-coloured
structures and define Boolean queries by selecting nodes. 
We
briefly introduce these languages and compare them with GNNs.

\paragraph{Coloured graphs.}
A \emph{$k$-colouring} of a graph $G=(V,E)$ is a $k$-embedding
$c : V \to \{0,1\}^k$, where $c_i(v)=1$ indicates that $v$ has colour~$i$.
A \emph{$k$-coloured} graph is a pair $(G,c)$, and a \emph{pointed} one is
$(G,c,v)$ with $v\in V$. Since colourings are special embeddings,
\mplang expressions and GNNs apply directly to coloured graphs. For
readability, we use colour names (e.g.\ \textsf{red}, \textsf{blue}) in place of
$P_1,P_2$.

\paragraph{Traditional query languages.}
First-Order Logic (FO) and Modal Logic (ML) are standard formalisms for
Boolean queries on coloured graphs. FO is interpreted over $k$-coloured graphs
with adjacency $E$ and unary predicates $P_i$, for $i \in [k]$,  indicating colours.
ML evaluates formulas at nodes using propositional symbols $P_i$, Boolean
connectives, and the modality $\mlDiamond$:
\[
(\mlDiamond \varphi)(G,c)(v)=1
\quad\text{iff}\quad
\exists u\,( \{u,v\}\in E \land \varphi(G,c)(u)=1 ).
\]
Graded Modal Logic (GML) extends ML with counting modalities
$\mlDiamond^{\ge n}$, where
\[
(\mlDiamond^{\ge n}\varphi)(G,c)(v)=1
\quad\text{iff}\quad
|\{u : \{u,v\}\in E,\ \varphi(G,c)(u)=1\}| \ge n.
\]
Boolean GML queries over coloured graphs are also expressible in FO, but not vice versa. 

\paragraph{Boolean expressiveness of GNNs.}
GNNs subsume GML for Boolean queries on coloured graphs.

\begin{theorem} 
{\em \cite{DBLP:conf/iclr/BarceloKM0RS20,DBLP:conf/icalp/BenediktLMT24}} 
\label{thm:well-known}
Fix $k > 1$. We have that GML is Boolean strictly 
contained in $\smplang{\sigma}$, for $\sigma \in \{\relu,\trrelu\}$, over the class of $k$-coloured graphs. 
\end{theorem}
\section{\amplang: No Activation Function}
\label{sec:amplang} 

To set a baseline for analysing the role of activation functions in the expressivity of \mplang, we first look at the purely \emph{affine} fragment, i.e., \textit{without} activation functions:
\amplang.
Since it is possible to write any affine function in \mplang without the use
of activation functions,
intuitively $\smplang{\sf id}$
is equivalent to
$\amplang$ (where ${\sf A}$ is the set of all affine functions),
hence the name.
By
\Cref{prop:gnn-is-mplang},  
this fragment captures GNNs that use only the identity activation {\sf id}. As we observe, even this seemingly weak setting can express Boolean queries beyond FO.

\begin{example}\label{ex:beyondFO}
Consider coloured graphs with node colours {\sf red} and {\sf blue}.  
The \amplang expression $(\mpDiamond {\sf red} - \mpDiamond {\sf blue})$ defines the
Boolean query that holds at a node iff it has more red than blue neighbours.
This query is not definable in FO, and thus neither in ML.
 \qed   
\end{example}

This illustrates that the expressive power of \amplang is subtler than it might initially appear.
{ 
In fact, a useful way to think about the Boolean fragment of \amplang
is to consider numerical \amplang expressions with one final bounded activation
that implements thresholding.}

\subsection{Numerical expressivity}
We show that \amplang{} can only compute numerical functions determined by walk counts and walk-summed features (over the class of all embedded graphs).  
To this aim, we derive a normal form of \amplang{} expressions in terms of walk counts
using the notion of  \textit{$\mpDiamond$-depth} of an expression:
\begin{itemize} 
\item 
The $\mpDiamond$-depth of the expressions
$1$ and $P_i$ are 0. 
\item 
Given two $\smplang{\Sigma}$ expressions $e,e'$
with $\mpDiamond$-depths $n$ and $n'$, respectively, 
the $\mpDiamond$-depth of $ae$ and $\sigma(e)$ is $n$,
the $\mpDiamond$-depth of $e+e'$ is $\text{max}\{n,n'\}$
and the $\mpDiamond$-depth of $\mpDiamond e$ is $n+1$.
\end{itemize}
Of particular interest for the normal form are expressions of the form 
$\mpDiamond^i 1$ and $\mpDiamond^i P_j$. When evaluated on a pointed embedded graph $(G,\gamma,v)$, these expressions correspond, respectively, 
to the number of walks of length $i$ starting at $v$, and to the sum, over all 
walks of length $i$ starting at $v$, of the $j$-th coordinate of the embedding 
at the endpoint of each walk. 
The normal form below shows that 
\textsf{A-MPLang} can extract information only through a simple form of walk statistics.

\begin{restatable}{theorem}{amplangnormalform}\label{amplang-normalform}
Let $e$ be an \amplang-expression over $d$-embeddings of $\mpDiamond$-depth $n$. Then, over the class of $d$-embedded graphs, $e$ is numerically 
equivalent to an expression of the form
\[
\sum_{i=0}^n \bigl( c_i\, \mpdDiamond^i 1 + \sum_{j=1}^{d} c_{i,j}\, \mpdDiamond^i P_j \bigr).
\]
\end{restatable}

This normal form can be proven by induction on the structure of \amplang{} expressions. 
Apart from delineating what type of graph structural information \amplang can extract, \Cref{amplang-normalform} can also be used to show that certain numerical functions cannot be expressed in \amplang.
In particular, \Cref{amplang-normalform} implies the following equivalence.  

\begin{corollary}\label{thm:walk-equivalence-corollary}
The following are equivalent for pointed $d$-embedded graphs
$(G_1,\gamma_1,v_1)$
and $(G_2, \gamma_2, v_2)$:
\begin{itemize}
    \item $e(G_1,\gamma_1)(v_1) = e(G_2,\gamma_2)(v_2)$, 
    for every \textsf{A-MPLang} expression $e$ of $\mpDiamond$-depth $n$ over $d$-embeddings. 
    \item $(\mpDiamond^i 1)(G_1,\gamma_1)(v_1) = (\mpDiamond^i 1)(G_2,\gamma_2)(v_2)$
    and
    $(\mpDiamond^i P_j)(G_1,\gamma_1)(v_1) = (\mpDiamond^i P_j)(G_2,\gamma_2)(v_2)$, for all $i \in [n]$ and $j \in [d]$. 
\end{itemize}
\end{corollary}
We refer to the { condition in the second item} by saying that $(G_1,v_1,\gamma_1)$ and $(G_2,v_2,\gamma_2)$ are \emph{$n$-walk equivalent}. Moreover, they are called \emph{walk equivalent} if they are $n$-walk equivalent for all $n \in \N$.
Hence, to show that a numerical query $f$ is not expressible in \amplang, 
we need to find two walk-equivalent 
pointed embedded graphs $(G_1,\gamma_1,v_1)$ and $(G_2,\gamma_2,v_2)$ with $f(G_1,\gamma_1)(v_1) \neq f(G_2,\gamma_2)(v_2)$. 

{
By \Cref{amplang-normalform}, every \amplang-expression is a linear combination of
walk statistics of the form $\mpDiamond^i 1$ and $\mpDiamond^i P_j$, and hence can
inspect the colours/features only at the \emph{endpoints} of walks of length $i$ 
from $v$. The next example exploits this: the query depends on a property of an
\emph{intermediate} node, which is
invisible to endpoint-only walk sums and therefore inexpressible in \amplang.
}

\begin{example}\label{ex:5paths}
Consider the path $P_5$ on five nodes $v_0,v_1,v_2,v_3 ,v_4$ with undirected edges { $\bigl\{\{v_0,v_1\},\allowbreak \{v_1,v_2\},\{v_2,v_3\},\{v_3,v_4\}\bigr\}$}, and let the distinguished node be $v_1$.
We consider two $2$-colourings (with colours {\sf red} and {\sf blue}) of this graph:
\begin{itemize}
    \item Colouring $c$ colours the nodes $v_0$ and $v_4$ {\sf blue}, and
    $v_1,v_2,v_3$ {\sf red}, that is, \raisebox{-0.5ex}{\includegraphics[height=3ex]{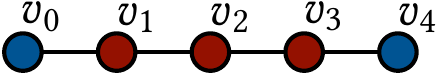}}.
    \item Colouring $c'$ colours $v_2$ {\sf blue}, and
    $v_0,v_1,v_3,v_4$ {\sf red}, that is,  \raisebox{-0.5ex}{\includegraphics[height=3ex]{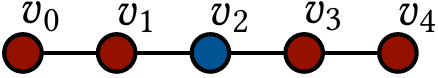}}.
\end{itemize}
We define a numerical 
query $f$ which counts the number of {\sf blue} neighbours of a node which have at least two red neighbours.
We have $f(P_5,c)(v_1)=0$ and $f(P_5,c')(v_1)=1$. Yet, $(P_5,v_1,c)$ and $(P_5,v_1,c')$ are 
walk-equivalent. {
Indeed, for $i\ge 0$ and $t\in\{0,1,2,3,4\}$ let $\mathsf{w}_i(t)$ be the number of walks of length $i$
from $v_1$ to $v_t$ in $P_5$. We note that
$\mathsf{w}_{i+1}(t)=\sum_{\{v_u,v_t\}\in E}\mathsf{w}_i(u)$.
For all $i\ge 0$, we have
$$
\mathsf{w}_i(2)=\mathsf{w}_i(0)+\mathsf{w}_i(4).
$$
For $i=0$ this is immediate and for $i>0$, using the adjacency of $P_5$ we have
$
\mathsf{w}_{i+1}(2)=\mathsf{w}_i(1)+\mathsf{w}_i(3)$,
and since $v_0$ (resp.\ $v_4$) has the unique neighbor $v_1$ (resp.\ $v_3$),
$\mathsf{w}_{i+1}(0)=\mathsf{w}_i(1)$ and  $\mathsf{w}_{i+1}(4)=\mathsf{w}_i(3)$.
Thus $\mathsf{w}_{i+1}(2)=\mathsf{w}_{i+1}(0)+\mathsf{w}_{i+1}(4)$, as desired. Now, for all $i$, $(\Diamond^i 1)(P_5,c)(v_1)=(\Diamond^i 1)(P_5,c')(v_1)$ since the underlying pointed graph is the same.
Moreover,
\[
(\Diamond^i{\sf blue})(P_5,c)(v_1)=\mathsf{w}_i(0)+\mathsf{w}_i(4),
\qquad
(\Diamond^i{\sf blue})(P_5,c')(v_1)=\mathsf{w}_i(2),
\]
which agree, as we have just shown. Finally ${\sf red}=1-{\sf blue}$ holds pointwise in both colourings,
so $\Diamond^i{\sf red}$ also agrees. Hence the two pointed structures are walk-equivalent.
}
Hence, $f$ is not expressible in \amplang{}  (over 2-coloured graphs). \qed
\end{example}

This example leverages that \amplang can only inspect the colours of terminal nodes along walks,
and computing the query relies on a property
of the middle node in walks of length 2.

\subsection{Boolean expressivity}
We know from \Cref{ex:beyondFO} that \amplang{} can express Boolean queries beyond FO. We now ask whether it forms a reasonable Boolean query language, which would at the very least require closure under the Boolean operations.
In the following,
we refer to the class of Boolean queries
expressible in (rational) $\amplang$
as (rational) $\amplang_\mathbb{B}$.
Our first result is positive.

\begin{proposition}\label{amplang-negation-closure}
    Rational $\amplang_\mathbb{B}$ on $k$-coloured graphs
    is closed under negation.
\end{proposition}

\begin{proof}
    Let $e$ be a rational \textsf{A-MPLang} expression of $\mpDiamond$-depth $n$ on 
$k$-coloured graphs. By the proof of \Cref{amplang-normalform}, $e$ is of the form
$\sum_{i=0}^n \Bigl(c_i\,\mpdDiamond^i 1 \;+\; \sum_{j=1}^k c_{i,j}\,\mpdDiamond^i P_j\Bigr)$, 
for suitable rational coefficients $c_i,c_{i,j}$. Let $d_e$ be a common denominator of 
all these coefficients. Since each $P_j$ is Boolean, every value $e(G,c)(v)$ is an integer 
multiple of $1/d_e$, so any positive value is at least $1/d_e$.
Set $k_e \coloneqq \nicefrac{1}{2d_e}$ and define $e' := k_e 1 - e$. Then, for a pointed 
$k$-coloured graph $(G,c,v)$:
\begin{itemize}
    \item if $e(G,c)(v)>0$, then $e'(G,c)(v)<0$;
    \item if $e(G,c)(v)\le 0$, then $e'(G,c)(v)\ge k_e>0$.
\end{itemize}
Thus $e'$ expresses the negation of the Boolean query defined by $e$.
\end{proof}

But this is as far as rational $\amplang_\mathbb{B}$  goes in behaving like a Boolean query language.

\begin{restatable}{proposition}{amplangnonclosure}\label{prop:nonclosure}
    Rational $\amplang_\mathbb{B}$ on $k$-coloured graphs is neither closed
    under conjunction nor under the Boolean $\mlDiamond$-modality.
\end{restatable}

In other words, even if $\phi_1$ and $\phi_2$ are in rational $\amplang_\mathbb{B}$, the Boolean queries $\phi_1 \land \phi_2$ and $\mlDiamond \phi_1$ need not be expressible in this way.
\begin{itemize} 
\item 
The proof of non-closure under conjunction uses the 2-coloured trees $(T,c_1)$, $(T,c_2)$, and $(T,c_3)$ (shown on the left in \Cref{fig:nonclosure}) together with a Boolean conjunction $\phi$. We show that any \amplang{} expression $e$ must satisfy either $e(T,c_1)(v) \geq e(T,c_2)(v)$ or $e(T,c_3)(v) \geq  e(T,c_2)(v)$, contradicting that $\phi$ holds on $(T,c_2,v)$ but fails on both $(T,c_1,v)$ and $(T,c_3,v)$.
\item 
We show non-closure under the Boolean $\mlDiamond$-modality by exhibiting pointed 2-coloured graphs $(G,c,v)$ and $(H,c',w)$ (depicted on the right in \Cref{fig:nonclosure}) that are walk-equivalent but can be distinguished by Boolean queries $\phi_1$ (in ML) and $\phi_2$ (in GML). Each of these queries applies $\mlDiamond$ to a Boolean query that is itself in $\amplang_\mathbb{B}$.
\end{itemize} 

\begin{figure}  
    \centering
    \includegraphics[height=2.4cm]{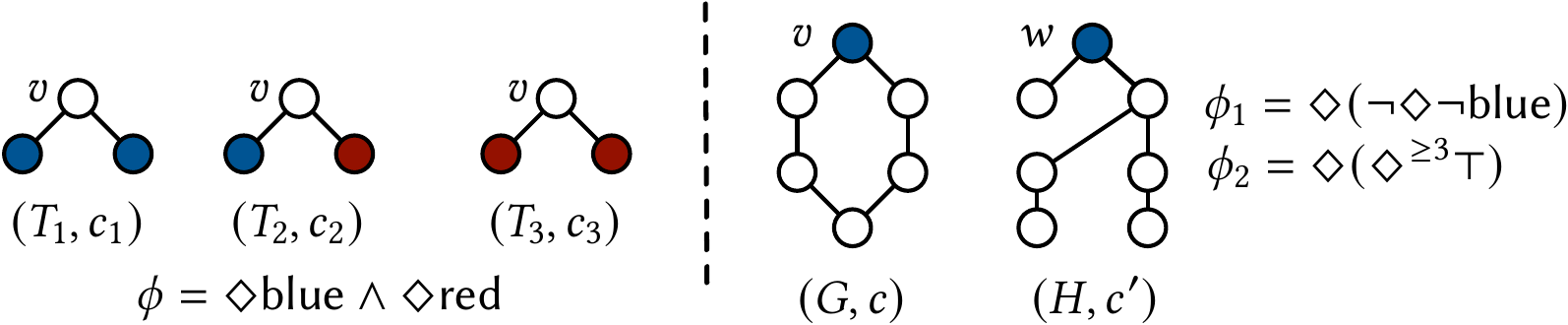}
    \caption{Graphs and Boolean queries used to show non-closure of $\amplang_\mathbb{B}$  in the proof of \Cref{prop:nonclosure}.}
    \label{fig:nonclosure}
\end{figure}

These negative results continue to hold for \amplang{} when we drop the restriction to rational coefficients. However, the proof of closure under negation relies in an essential way on having only rational coefficients. Once we lift this assumption, even that remaining Boolean closure property is lost.\looseness=-1

\begin{restatable}{proposition}{amplangnotclosedundernegation}
\label{prop:amplang-not-cloused-under-negation}
    $\amplang_\mathbb{B}$ on $k$-coloured graphs is not closed under negation.
\end{restatable}

In particular, our proof demonstrates that the query $e := r\mpDiamond \mathsf{red} - \mpDiamond\mathsf{blue}$ has no negation when $r$ is irrational. In fact, our proof requires $r$ to be a Liouville number, a special type of non-algebraic real that can be approximated particularly well in terms of rational numbers. On a technical level this leads to $e$ becoming arbitrarily close to $0$ at a very fast rate, making it analytically impossible for a hypothetical \mplang expression that negates $e$ to separate between cases where $e$ evaluates to $0$ or positive values. It remains an interesting open question whether non-closure under negation can  also be observed with algebraic irrational coefficients.

\begin{remark} \label{remark:incompML}
The previous results, together with {\Cref{ex:beyondFO}}, show that
$\amplang_\mathbb{B}$ and ML are incomparable in expressive power on
coloured graphs. On uncoloured graphs, however, every ML formula is
equivalent to one of $0,1,\mlDiamond 1$, or $\neg\mlDiamond 1$ and is then
subsumed by $\amplang_\mathbb{B}$. In contrast, the GML formula $\phi_2$ in
Figure~\ref{fig:nonclosure} is not expressible in $\amplang_\mathbb{B}$ even on
uncoloured graphs. In conclusion, Theorem \ref{thm:well-known} holds only in the presence of non-linearities in GNNs. 
 \qed
\end{remark} 

As we now show, activation functions restore closure under Boolean operations. In fact, \textsf{MPLang} with any eventually constant activation is closed under Boolean operators and subsumes earlier logical frameworks for characterizing $\relu$-GNNs.

\section{Bounded Activation Functions}
\label{sec:bounded}

Having established a baseline for the numerical and Boolean expressive power
of \mplang, we now investigate what happens when we equip the language with
bounded non-linear activation functions.
We will start with the activation function \step,
show that it is canonical,
and then connect it to related work on the logical expressivity
of GNNs with eventually constant activation functions.

\subsection{MPLang with Booleanisation}
Our guiding theme here is \emph{Booleanisation}: using activations to turn numerical \mplang-expressions into robust Boolean query languages. For this, we introduce the activation function \step:
\[
\step(x)=
\begin{cases}
1 & \text{if $x>0$,}\\[2pt]
0 & \text{if $x\le 0$.}
\end{cases}
\]
This is the indicator of the positive reals, also known as the \emph{unit step function}. It bridges numerical and Boolean \mplang queries: the map $e \mapsto \step(e)$ turns any numerical expression into a Boolean one.

Following the notation $\smplang{\Sigma}$ introduced in Section~\ref{subsec:mplang}, the expressions that can use $\step$ as activation function is denoted by
$\smplang{\step}$.  Note that this is still a numerical query language;
we write $\smplang{\step}_\mathbb{B}$
for the class of resulting Boolean queries. As we see next, this activation function already yields strictly more expressive power than \amplang.

\begin{example}
Consider the Boolean expressions 
$\phi$, $\phi_1$ and $\phi_2$ used in the proof of
\Cref{prop:nonclosure} and depicted in \Cref{fig:nonclosure}. These Boolean
queries are not expressible in \amplang. However, they become definable
once we allow the \step activation.
Indeed, it is readily verified that the ML formula 
$\phi = \mlDiamond{\sf red} \land \mlDiamond{\sf blue}$
can be expressed in \smplang{$\step$} as
$e=\step(\mpDiamond{\sf red}) + \step(\mpDiamond{\sf blue}) - 1$,
that
$\phi_1 = \mlDiamond(\neg \mlDiamond \neg {\sf blue})$
can be expressed as
$e_1=\mpDiamond\bigl(1 - \step(\mpDiamond(1-{\sf blue}))\bigr)$, and that
$\phi_2 = \mlDiamond(\mlDiamond^{\geq 3}\top)$
can be expressed as $
e_2=\mpDiamond\bigl(\step((\mpDiamond 1)-2)\bigr)$.
As a final example, we observe that
\[
    \mpDiamond \step \big({\sf blue} + (\step(\mpDiamond {\sf red}-1)) -1 \big)
\]
returns the number of blue neighbours of a node
which have at least two red neighbours. 
This computes the function from \Cref{ex:5paths}
that was not computable using \amplang. Hence,
$\smplang{\stepb}$ is more expressive than \amplang
on both numerical and Boolean queries.\qed
\end{example}

More generally, we prove that $\smplang{$\stepb$}_{\mathbb{B}}$ is closed under Boolean operations.

\begin{proposition}
$\smplang{$\stepb$}_{\mathbb{B}}$ is closed under negation and disjunction. 
\end{proposition}

\begin{proof} Recall that a Boolean query $\phi$ is expressed by an expression $e$ if
$\phi$ holds at $v$ on $(G,\gamma)$ iff $e(G,\gamma)(v) > 0$, for every pointed
embedded graph $(G,\gamma,v)$. Let $\phi_1$ and $\phi_2$ be Boolean queries
expressed by $e_1$ and $e_2$ in \smplang{$\step$}. For $\phi = \neg \phi_1$ we
set $e \coloneqq 1 - \step(e_1)$, and for $\phi = \phi_1 \vee \phi_2$ we set
$e \coloneqq \step(e_1) + \step(e_2)$.
\end{proof}

We can therefore treat $\smplang{$\stepb$}_{\mathbb{B}}$ as a genuine Boolean query language:
whenever an expression $e$ defines a Boolean query via the threshold
condition $e(G,\gamma)(v) > 0$, we may freely build new Boolean queries from
it using the usual logical connectives and comparisons, and stay within
$\smplang{$\stepb$}_{\mathbb{B}}$. More precisely, for \smplang{$\step$}-expressions $e_1$ and $e_2$, we abbreviate:
\begin{gather*}  
e_1 \vee e_2 \coloneqq \step(e_1)+\step(e_2),\qquad
e_1 \land e_2 \coloneqq \step(e_1)+\step(e_2)-1,\qquad
\neg e_1 \coloneqq 1-\step(e_1),\\[3pt]
e_1 > e_2 \coloneqq \step(e_1-e_2),\qquad
e_1 \leq e_2 \coloneqq \neg(e_1>e_2),\qquad
e_1 = e_2 \coloneqq \neg\bigl(e_1>e_2 \,\vee\, e_2>e_1\bigr),
\end{gather*}
and similarly for the remaining comparison operators.

\begin{example}
These shorthand notations make it easy to specify more involved Boolean
queries directly at the level of \smplang{$\step$}. For instance, consider
\[
(\mpDiamond{\sf blue} > \mpDiamond{\sf red} \wedge \mpDiamond{\sf white} = 0)
\vee
(\mpDiamond{\sf white} + \mpDiamond{\sf blue} \leq \mpDiamond{\sf red}),
\]
which says that either a node has strictly more blue than red neighbours and
no white neighbours, or the total number of white and blue neighbours does not
exceed the number of red neighbours.
\qed
\end{example}

The syntactic sugar will also make it very clear
to see the connections to other logical frameworks
for GNNs with eventually constant activation functions.
To that end, we first show that the choice of $\bool$
as our non-linear bounded activation function is not arbitrary.

\subsection{Eventually constant activation functions}
We show that \smplang{$\step$} acts as a canonical language, capturing the expressive power of various \mplang variants with different activation functions over coloured graphs. Our focus is on eventually constant activations, a standard choice in the  study of neural networks \cite{DBLP:conf/icalp/BenediktLMT24,siegelmannsontag}.

\begin{definition}[Eventually constant function]
An \emph{eventually constant function} is a function $\sigma:\R\to\R$ for which
there exist $t^-<t^+$ such that $\sigma(t)=\sigma(t^-)$ for all $t\le t^-$ and
$\sigma(t)=\sigma(t^+)$ for all $t\ge t^+$. If $\sigma(t^-)\neq\sigma(t^+)$, we
call $\sigma$ \emph{uneven eventually constant}.
\end{definition}

The Booleanisation function $\step$ is an uneven eventually constant activation; others include $\trrelu$, $\sign$, and hard versions of {\sf sigmoid} and ${\sf tanh}$. We are not aware of any eventually constant activation used in machine learning that is not uneven. We show that, with rational coefficients, the specific choice of uneven eventually constant activation does not affect expressive power.

\begin{restatable}{theorem}{eventuallyconstantfunctions}
\label{thm:eventually-constant-activation-functions}
    Let $\Sigma$ be a finite set of uneven eventually constant functions. Then the rational fragment of  \smplang{$\,\Sigma$} is numerically 
    equivalent to \smplang{$\step$} on $k$-coloured graphs.
\end{restatable}
\begin{figure}
\includegraphics[width=\textwidth]{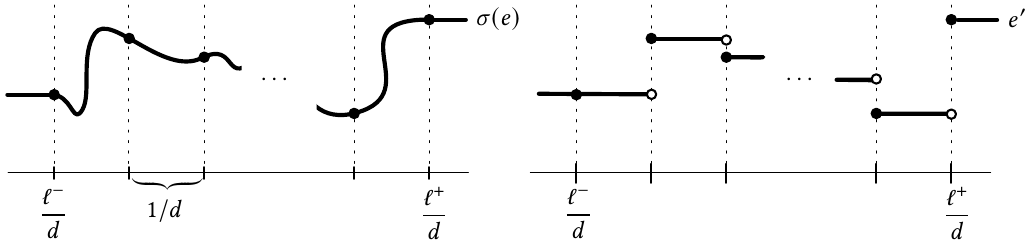}
\caption{Simulation of $\sigma(e)$ by a \smplang{$\bool$} expression $e'$ as described in the proof of \Cref{thm:eventually-constant-activation-functions}.}\label{fig:evconst}
\end{figure}
\begin{proof}[Proof sketch]
The proof relies on the fact that any rational \smplang{$\sigma$}-expression on
$k$-coloured graphs can take only finitely many values on a fixed interval,
which can be encoded in \smplang{$\step$} using suitable affine
transformations. To see this, note that if $d$
is the lowest common denominator in a rational \mplang expression $e$,
then on any pointed coloured graph $(G,c,v)$,
the evaluation $e(G,c)(v)$ only takes values of the form $\nicefrac{\ell}{d}$ for $\ell \in \mathbb{Z}$. It  thus suffices to simulate $\sigma(e)$ on that finite set only, which can be done by appropriately rescaling and combining $\bool$-functions applied on $e$, as shown in~\Cref{fig:evconst}.
\end{proof}

 This strictly extends the result of
Benedikt et al.~\cite{DBLP:conf/icalp/BenediktLMT24}, who established a similar
correspondence  for Boolean queries computed by GNNs with eventually
constant activations \emph{but without linear layers}. In contrast, our setting
allows linear layers and numerical (not only Boolean) queries, and therefore goes beyond the expressive
scope considered in their work.

\begin{remark} \label{limitremark}
    If $\sigma$ has limits at $\pm\infty$, then for any $\varepsilon>0$ any
\smplang{$\sigma$}-expression can be $\varepsilon$-approximated on all
$k$-coloured graphs by a \smplang{$\step$}-expression. This paves the way for
quantitative approximation results for \mplang with bounded activations.
    \qed
\end{remark}

\subsection{Connections with logics for Boolean GNN expressivity}

Using $\smplang{\step}$ as a stand-in for \mplang equipped
with any eventually constant function,
we can now see how the Boolean fragment $\smplang{\step}_\mathbb{B}$
connects directly to modal logics with
counting and Presburger constraints that have recently been used to characterise
GNNs \cite{nunn2024logic, DBLP:conf/icalp/BenediktLMT24}.\footnote{The logic K\# was independently introduced by Nunn et al. in
\cite{nunn2024logic}, and also captures the Boolean expressivity
of \trrelu-GNNs without linear layers.
However, the setting in this work is restricted to integer coefficients.}

Benedikt et al.\ \cite{DBLP:conf/icalp/BenediktLMT24},
introduces \emph{local modal logics with Presburger quantifiers}, such as
$\mathcal{L}\text{-MP}^2$.
 This logic extends ML with
numerical atoms of the form
\[
  \sum_{i} \lambda_i \cdot \#y\,\bigl[E(x,y)\wedge \varphi_i(y)\bigr] \,\mathsf{ op}\, \delta, 
\]
where the $\varphi_i$ are themselves formulas
in $\mathcal{L}\text{-MP}^2$, 
$\lambda_i,\delta \in \Q$, and $\mathsf{op}\in\{=,\neq,\leq,\geq,<,>\}$. Such atoms
evaluate to true or false depending on whether the indicated linear inequality
is satisfied by the counts of one-hop neighbours whose endpoints satisfy
the respective formulas. 

\paragraph{$\mathcal{L}\text{-MP}^2$ is strictly contained in $\smplang{\step}_\mathbb{B}$}
Using the shorthands we developed
for the language 
$\smplang{\step}$, it is easy to see the connection:

\begin{proposition}
    \label{prop:compare.presburger}
         $\mathcal{L}\text{-MP}^2$ 
        is Boolean strictly contained in $\smplang{\step}$ over $k$-coloured graphs (strict containment holds 
        even when $k = 2$).
\end{proposition}

\begin{proof}
        We first show containment. The Presburger quantifiers
        can be translated to 
        \[\sum_{i} c_{i}\,\mpdDiamond e_i \text{ op } \delta,
        \]
        where
        $c_i,\delta \in \Q$ and $\mathsf{op}\in\{=,\neq,\leq,\geq,<,>\}$,
        which can be translated using our shorthand for inequalities.
        Also recall the closure of $\smplang{\step}_\mathbb{B}$ under Boolean combinations.
        The $P_i$ represent unary predicates
        and $1$ is the formula that is always true.

          We now show that the containment is strict. Note that $\smplang{\step}_\mathbb{B}$ extends $\mathcal{L}\text{-MP}^2$ by allowing arbitrarily many hops, precisely the extra expressive power enabled by the identity function. This is illustrated by the numerical query: 
\[
  \mpdDiamond\mpdDiamond{\sf green} = \mpdDiamond\mpdDiamond{\sf blue},
\]
which asks whether a node has the same number of $\sf green$ and $\sf blue$
grandchildren. Since this is an equivalence of two \amplang expressions, it is directly recognisable as expressible in \smplang{$\step$}. But
it is known that this query is inexpressible in $\mathcal{L}$-$\mathsf{MP}^2$ 
\cite{DBLP:conf/icalp/BenediktLMT24}. 
\end{proof}
    
   This suggests that the Boolean queries definable by $\trrelu$-GNNs—equivalent, over coloured graphs, to those of $\mathcal{L}\text{-MP}^2$ \cite{DBLP:conf/icalp/BenediktLMT24}—match a specific syntactic subfragment of $\smplang{\step}$.

\paragraph{Linear layers add expressive power.}
By \Cref{thm:eventually-constant-activation-functions} and \Cref{prop:gnn-is-mplang},
\smplang{$\step$} is expressively equivalent to any $\{\sigma,\mathsf{id}\}$-GNN with $\sigma$
eventually constant. Likewise, $\mathcal{L}$-MP$^2$ is known to capture $\sigma$-GNNs for such
$\sigma$~\cite{DBLP:conf/icalp/BenediktLMT24}. Our comparison therefore shows that adding the
identity activation strictly increases expressivity for uneven eventually constant activations (even for Boolean queries).

\begin{proposition}\label{morepower}
    Let $\Sigma,\Sigma'$ be sets of uneven eventually constant activation functions.
    Then $\Sigma'$-GNNs are Boolean strictly contained in $(\Sigma \cup \{\mathsf{id}\})$-GNNs over $k$-coloured graphs (even for $k = 2$).
\end{proposition}

\Cref{thm:eventually-constant-activation-functions} and \Cref{morepower} are the exact type of results that motivate the study of GNNs from the perspective of query languages. We see how the formal study of these query languages provides real insight for the design of GNNs. Concretely, these two results show that combining multiple activation functions (of the form in \Cref{thm:eventually-constant-activation-functions}) cannot increase expressivity. That is, a $\{\trrelu, \sign, \mathsf{id}\}$-GNN is just as powerful as a $\{\trrelu, \mathsf{id}\}$-GNN. Yet, the presence of linear layers in GNNs does increase expressivity.

\section{Beyond Bounded Activation Functions}
\label{sec:relu}

A natural next question is how the expressive power of \mplang changes once we
allow \emph{unbounded} activation functions. In what follows, we fix the
activation to be $\relu$. Recall from \Cref{prop:gnn-is-mplang} that
\smplang{$\relu$} captures exactly $\relu$-GNNs, and more generally
\smplang{$\sigma$} captures $\{\mathsf{id},\sigma\}$-GNNs. Thus, the question we are facing is:

\medskip
\emph{How do $\relu$-GNNs and $\{\mathsf{id},\sigma\}$-GNNs compare in expressive
power, when $\sigma$ ranges over eventually constant activation functions?}
\medskip

This turns out to be a delicate problem. For Boolean queries, Benedikt et
al.\ \cite{DBLP:conf/icalp/BenediktLMT24} showed that, on $k$-coloured graphs,
rational $\relu$-GNNs are strictly more expressive than rational
$\{\sigma\}$-GNNs. However, by \Cref{morepower}, \smplang{$\sigma$} (or, equivalently,
$\{\sigma,\mathsf{id}\}$-GNNs) is strictly more expressive than
$\{\sigma\}$-GNNs, so the comparison between $\relu$-GNNs and
$\{\sigma,\mathsf{id}\}$-GNNs requires a fresh analysis.

Our main result shows that adding $\relu$ increases the \emph{numerical}
expressive power of \mplang.

\begin{theorem}\label{thm:relusep}
For any set of uneven eventually constant activation functions $\Sigma$, rational
\smplang{$\relu$} numerically strictly contains 
rational \smplang{$\Sigma$}
on $k$-coloured graphs. 
\end{theorem}

The proof of this theorem hinges on the following stronger separation, which
does not rely on rational coefficients.

\begin{proposition}\label{reluvsstep}
\smplang{$\relu$} numerically strictly contains \smplang{$\step$} on $k$-coloured graphs.
\end{proposition}
Combined with \Cref{thm:eventually-constant-activation-functions}, which shows
that all eventually constant activation functions are equivalent
to \bool, the proposition yields \Cref{thm:relusep} immediately.

The following important corollary of Proposition \ref{reluvsstep} constitutes an important, non-trivial extension of the aforementioned result by \citeauthor{DBLP:conf/icalp/BenediktLMT24} stating that \trrelu-GNNs are numerically strictly contained in \relu-GNNs \cite{DBLP:conf/icalp/BenediktLMT24}. 

\begin{corollary}
\relu-GNNs numerically strictly contain $\{\trrelu,{\sf id}\}$-GNNs on $k$-coloured graphs.
\end{corollary}

The following subsection is dedicated to the proof of \Cref{reluvsstep}.

\subsection{Proof of \Cref{reluvsstep}}

We start from a very simple numerical query, definable in $\smplang{\relu}$:
\[
  Q  \coloneqq  \relu\bigl(\mpDiamond {\sf red} - \mpDiamond {\sf blue}\bigr),
\]
which counts how many more red than blue neighbours a node has, truncated at $0$.
Our aim is to show that $Q$ is \emph{not} definable in $\smplang{\step}$. This
will show the desired strictness in \Cref{reluvsstep}.

\paragraph{Red--blue symmetric trees}
Our argument uses a family of highly symmetric trees whose root can
``see'' the imbalance between red and blue children through the numerical
power of $\relu$.
{ For readability, we  introduce white nodes as a third colour,
but note that the following proof is exactly analogous
if all of the white nodes in the construction get coloured red.
In total, we only require two colours, which can be realized by a $1$-colouring
(which assigns $1$ to red nodes, $0$ to blue nodes).}

Fix  $r,b,k \in \N$.
We write $T[r,b,k]$ for the \emph{red--blue symmetric tree}
with parameters $(r,b,k)$.
To define $T[r,b,k]$, it is convenient to first introduce a basic gadget.
Let $H[r,b]$ be the star graph consisting of a single white node with $r$ red
children and $b$ blue children.
Intuitively, $T[r,b,k]$ is obtained by gluing together copies of $H[r,b]$
and $H[b,r]$ in such a way that red and blue nodes always ``sit between''
two white nodes of different types.
Formally, $T[r,b,k]$ is the unique finite rooted tree of height $k$ satisfying the following properties.
There are four kinds of nodes:
\begin{itemize}
  \item \emph{$w$-type} white nodes: they have exactly $r$ red and $b$ blue neighbours;
  \item \emph{$w'$-type} white nodes: they have exactly $b$ red and $r$ blue neighbours;
  \item red nodes: they have exactly one $w$-type and one $w'$-type white neighbour;
  \item blue nodes: they have exactly one $w$-type and one $w'$-type white neighbour.
\end{itemize}
Moreover, the root is a $w$-type white node. We remark that all nodes lie at distance at most
$k$ from the root. See \Cref{fig:redbluetree}
for the case $k=3$.
\begin{figure}
  \centering
  \includegraphics[width=\textwidth]{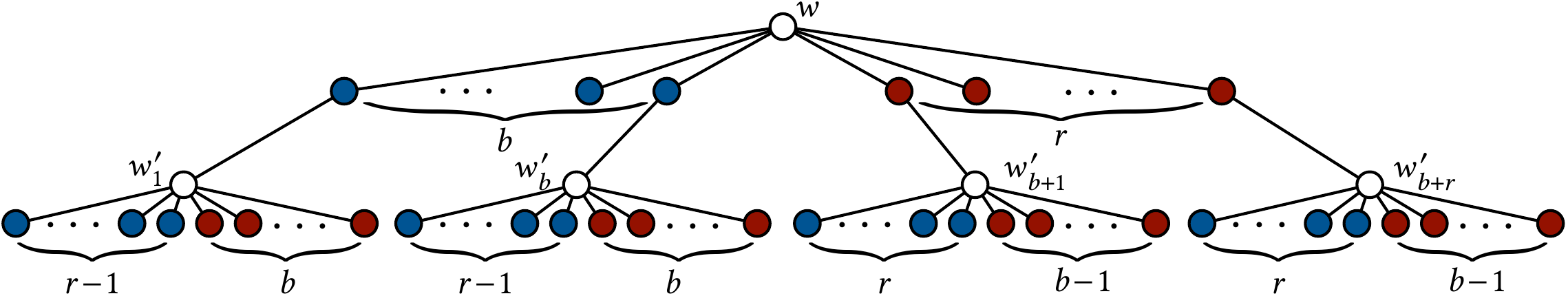}
  \caption{The red--blue symmetric tree $T[r,b,3]$.}
  \label{fig:redbluetree}
\end{figure}

The parameter $k$ will always be chosen to bound the $\mpDiamond$-depth of the
expressions we evaluate: if $e$ has $\mpDiamond$-depth at most $k$, then at each
node $v$ the value $e(T[r,b,k])(v)$ only depends on the radius-$k$
neighbourhood of $v$, which is completely contained in $T[r,b,k]$ and
isomorphic for all nodes of the same type.
Whenever $k$ is fixed and understood from context, we write
\[
  e(v) \;\coloneqq\; e\bigl(T[r,b,k]\bigr)(v)
\]
for the value of $e$ at a node $v$ of $T[r,b,k]$.

\paragraph{Nice functions on $(r,b)$}
We now describe the shape of the numerical functions
that arise when evaluating $\smplang{\step}$-expressions on $T[r,b,k]$.
The key point is that, because $\step$ is eventually constant and the tree
is highly symmetric, the dependence on $(r,b)$ is very tame.

\begin{definition}[Simple and symmetric functions]
  A function $\beta : \R^n \to \R$ is called \emph{simple}
  if there exist real constants $a_1,\dots,a_m$ and a partition
  $\R^n = I_1 \,\dot\cup\, \dots \,\dot\cup\, I_m$ such that
  \[
    \beta(\vec x)   =   a_j \qquad\text{for all }\vec x \in I_j.
  \]
  A function $s : \R^2 \to \R$ is \emph{symmetric}
  if $s(x,y) = s(y,x)$ for all $x,y \in \R$. \qed
\end{definition}

Thus simple functions take only finitely many values,
and symmetric functions are invariant under swapping their arguments.

\begin{definition}[Nice functions]
  A function $F : \R^2 \to \R$ is called \emph{nice}
  if it can be written as
  \[
    F(x,y)   =   \sum_{i=1}^m p_i(x,y)\,\beta_i(x,y),
  \]
  where each $p_i$ is a polynomial and each $\beta_i$ is simple and symmetric. \qed 
\end{definition}

We observe that nice functions are closed under addition and
multiplication by polynomials.

\paragraph{What \texorpdfstring{$\smplang{\step}$}{S-MPLang(step)} can see on $T[r,b,k]$}

We now show that, on the red--blue symmetric trees $T[r,b,k]$,
every $\smplang{\step}$-expression has an extremely rigid form.

\begin{restatable}{lemma}{boolmplangonrbgraph}\label{thm:stepmplang-on-rbgraph}
  Let $e$ be a $\smplang{\step}$-expression of $\mpDiamond$-depth at most $k$.
  Then there exist
  \begin{itemize}
    \item nice functions $F,S_r,S_b : \R^2 \to \R$, where $S_r$ and $S_b$ are symmetric,
    \item an affine function $\alpha : \R^2 \to \R$, and
    \item a simple function $\beta : \R^2 \to \R$,
  \end{itemize}
  such that, for all $r,b \in \N$ and all nodes $v$ of $T[r,b,k]$:
  \begin{itemize}
    \item if $v$ is red, then $e(v) = S_r(r,b)$;
    \item if $v$ is blue, then $e(v) = S_b(r,b)$;
    \item if $v$ is a $w$-type white node, then
      $e(v) = F(r,b) + \alpha(r,b) + \beta(r,b)$;
    \item if $v$ is a $w'$-type white node, then
      $e(v) = F(b,r) + \alpha(b,r) + \beta(b,r)$.
  \end{itemize}
\end{restatable}
\begin{proof}[Proof sketch]
  The proof is by structural induction on $e$.
  The base cases are easy:
  constants and colour predicates $P_i$ clearly give rise to constant functions,
  which are both nice and symmetric. The inductive steps for scalar
  multiplication and addition use the closure of nice, affine and simple functions
  under these operations.
  The only interesting cases are the applications of $\step$ and of $\mpDiamond$.
  For $e = \step(e')$, the induction hypothesis provides functions of the desired
  shape for $e'$. Applying $\step$ then simply replaces these functions by simple
  ones: on each region where $e'$ is constant, $\step(e')$ is also constant.
  Symmetry is preserved because $\step$ is applied pointwise.
  For $e = \mpDiamond e'$, we distinguish the four node types.
  Red and blue nodes each have one $w$-type and one $w'$-type neighbour.
  The value of $\mpDiamond e'$ at such a node is therefore the sum
  of the values at these two neighbours, which yields a symmetric nice function
  in $(r,b)$ by the induction hypothesis and closure properties.
  On a $w$-type node, the value of $\mpDiamond e'$ is a linear combination of
  $S_r(r,b)$ and $S_b(r,b)$ with coefficients $r$ and $b$. This again yields a nice function of $(r,b)$.
  The case of $w'$-type nodes is obtained from this by swapping $r$ and $b$,
  which explains the occurrence of $F(b,r)$, $\alpha(b,r)$ and $\beta(b,r)$.
  A straightforward bookkeeping of the affine and simple parts gives the
  desired shape.
\end{proof}

In particular, evaluating $e$ at the root of $T[r,b,k]$ always yields a value
of the form
\[
  e(\text{root}) = F(r,b) + \alpha(r,b) + \beta(r,b), 
\]
for some nice $F$, affine $\alpha$, and simple~$\beta$ depending only on $e$. We recall that
the parameter $k$ bounds the $\mpDiamond$-depth of $e$, ensuring the root ``sees''
precisely the portion of the tree we control.

\paragraph{Nice functions cannot encode \texorpdfstring{$|r-b|$}{|r−b|}}
The next lemma is the technical heart of the argument: 
functions of the above form cannot coincide with
the absolute difference $|r-b|$ over the integers.

\begin{restatable}{lemma}{blobinomialsvsrelu}\label{thm:blobinomials-vs-relu}
  Let $F : \R^2 \to \R$ be nice, let $\alpha : \R^2 \to \R$ be affine,
  and let $\beta : \R^2 \to \R$ be simple.
  Then there exist $x,y \in \N$ such that
  \[
    F(x,y) + \alpha(x,y) + \beta(x,y) \;\neq\; |x-y|.
  \]
\end{restatable}

\begin{proof}[Proof sketch]
Write $F(x,y)=\sum_{j=1}^k p_j(x,y)\,\beta_j(x,y)$ with $p_j$ polynomials and
$\beta_j$ simple symmetric. Assume for contradiction that
\[
G(x,y)\coloneqq F(x,y)+\alpha(x,y)+\beta(x,y)=|x-y|
\quad\text{for all }(x,y)\in\mathbb{N}^2.
\]
Fix $\delta\in\mathbb{N}$ and consider the diagonal
\[
L_\delta\coloneqq\{(m+\delta,m)\mid m\in\mathbb{N}\}.
\]
Along $L_\delta$ we have $G(m+\delta,m)=\delta$, while the tuple
\[
\bigl(\beta_1(m+\delta,m),\dots,\beta_k(m+\delta,m),\beta(m+\delta,m)\bigr)
\]
takes only finitely many values. Hence, on an infinite subset of $L_\delta$,
all these simple functions are constant: there exist real constants
$c_1(\delta),\dots,c_k(\delta)$ and $b_\delta$ such that
$\beta_j(m+\delta,m)=c_j(\delta)$ for all $j$ and
$\beta(m+\delta,m)=b_\delta$ on that subset. On this subset, $G$ reduces to a
single polynomial
\[
Q_\delta(r,b)=p_0(r,b)+\sum_{j=1}^k c_j(\delta)\,p_j(r,b)
\]
satisfying $Q_\delta(m+\delta,m)+b_\delta=\delta$, hence
\[
Q_\delta(x+\delta,x)\equiv \delta-b_\delta,
\]
where $\equiv$ denotes equality of polynomials on the entire domain. Using the
symmetry of the $\beta_j$ and again the finiteness of the range of $\beta$, we
similarly obtain
\[
Q_\delta(x,x+\delta)\equiv \delta-b'_\delta
\]
for some constant $b'_\delta$.

Also when $\delta$ varies, the finite tuple
$(c_1(\delta),\dots,c_k(\delta),b_\delta,b'_\delta)$
can take only finitely many values. Thus there exists an infinite set
$D\subseteq\mathbb{N}$ and fixed constants $c_1,\dots,c_k,b,b'$ such that for
all $d\in D$ the corresponding polynomial $Q_d$ equals a single polynomial
\[
Q(r,b):=p_0(r,b)+\sum_{j=1}^k c_j\,p_j(r,b),
\]
and satisfies
\[
Q(x+d,x)\equiv d-b,
\qquad
Q(x,x+d)\equiv d-b'.
\]
Set $u=x-y$ and $v=y$ and define
$\widetilde Q(u,v):=Q(u+v,v)$. For each $d\in D$ we have
$\widetilde Q(d,v)=Q(v+d,v)\equiv d-b$, which is constant in $v$. Hence
$\widetilde Q(u,v)$ does not depend on $v$, so there is a univariate polynomial
$C$ such that $Q(x,y)=C(x-y)$ depends only on $x-y$. For $d\in D$ this gives
$C(d)=d-b$, and using $Q(x,x+d)\equiv d-b'$ also $C(-d)=d-b'$. Thus
\[
C(d)+C(-d)=2d-(b+b').
\]
The left-hand side is an even polynomial in $d$, whereas the right-hand side is
not. The equality holds for infinitely many $d$ values, so they should be also equal as polynomials, which is impossible.
This is our contradiction. 
\end{proof}

\paragraph{ReLU beats eventually constant activations}
We can now put everything together.

\begin{theorem}\label{thm:relu-vs-trrelu}
  The following numerical query is not expressible in $\smplang{\step}$: 
  \[
    Q\coloneqq{\normalfont\relu}\bigl(\mpDiamond {\sf red} - \mpDiamond {\sf blue}\bigr).
  \]
\end{theorem}

\begin{proof}
  Suppose, towards a contradiction, that there is a \smplang{$\step$}-expression $e$
  defining the numerical query $Q$.
  By exchanging the roles of ${\sf red}$ and ${\sf blue}$ in $e$ we obtain
  another \smplang{$\step$}-expression $e'$ that defines the query $\relu\bigl(\mpDiamond {\sf blue} - \mpDiamond {\sf red}\bigr)$.  Consider now the \smplang{$\step$}-expression $e^* \coloneqq e- e'$.
  By construction, $e^*$ defines the absolute difference
  $\bigl|\mpDiamond {\sf red} - \mpDiamond {\sf blue}\bigr|$ at every node.
  Let $k$ be the $\mpDiamond$-depth of $e^*$. Evaluating $e^*$ at the root
  of $T[r,b,k]$ only depends on the radius-$k$ neighbourhood of the root,
  so \Cref{thm:stepmplang-on-rbgraph} applies: there exist
  a nice function $F$, an affine function $\alpha$ and a simple function
  $\beta$ such that
$e^*(\text{root of }T[r,b,k])=F(r,b) + \alpha(r,b) + \beta(r,b)$ for all $r,b \in \N$.
  On the other hand, by the definition of $T[r,b,k]$,
  the root has exactly $r$ red and $b$ blue neighbours, so
  $e^*(\text{root of }T[r,b,k]) = |r-b|$.  This contradicts \Cref{thm:blobinomials-vs-relu}. 
\end{proof}

\subsection{Comparison with the separation of \relu-GNNs from \trrelu-GNNs.} 
Let us once more clarify the relationship between 
our results in this section and the work by Benedikt et al.~\cite{DBLP:conf/icalp/BenediktLMT24}.  They separate \relu-GNNs from \trrelu-GNNs \emph{without} linear layers by showing that the query $\mpDiamond\!\mpDiamond{\sf green} = \mpDiamond\!\mpDiamond{\sf blue}$ is expressible by \relu-GNNs but not by \trrelu-GNNs, where they could restrict attention to
graphs of depth two. In contrast, our separation must handle identity (linear) layers, and as proven in Proposition~\ref{prop:compare.presburger}, we \emph{can} express said query in $\smplang{\step}$, thus with
$\{\trrelu,{\rm id}\}$-GNNs.  Moreover, on graphs of bounded depth, the query of Theorem~\ref{thm:relu-vs-trrelu} is in fact expressible (proof omitted).  This explains why our proof must consider graphs of arbitrary depth.\looseness=-1

\section{Conclusion and Open Problems}
\label{sec:final}

{
\begin{table}[t] 
\caption{Graphical summary of results by \emph{query type} (vertical) and \emph{input setting} (horizontal). Results highlighted in gray represent ideal placement (most restrictive for negative results, least restrictive for positive results).}\label{tbl:summary}
\centering
\footnotesize
\begin{tikzpicture}[
  idealbox/.style={
    draw, rounded corners,
    align=left,
    fill=black!3,   
    draw=black,
    inner xsep=3pt, inner ysep=2.5pt,
    font=\scriptsize,
    text width=5.35cm
  },
  box/.style={
    draw, rounded corners,
    align=left,
    inner xsep=3pt, inner ysep=2.5pt,
    font=\scriptsize,
    text width=5.35cm
  },
  axis/.style={font=\small}
]

\draw (0,-0.8) rectangle (12,6);
\draw (6,-.8) -- (6,6);
\draw (0,3) -- (12,3);

\node[axis] at (3,6.35) {$d$-embedded graphs};
\node[axis] at (9,6.35) {coloured graphs};
\node[axis, rotate=90] at (-0.6,4.5) {numerical};
\node[axis, rotate=90] at (-0.6,1.2) {Boolean};

\coordinate (TL) at (0.25,5.75);
\coordinate (TR) at (6.25,5.75);
\coordinate (BL) at (0.25,2.7);
\coordinate (BR) at (6.25,2.7);

\node[idealbox, anchor=north west] (nf) at (TL) {%
\textbf{A-MPLang} is equivalent to 
affine combinations of walk-summed features
(Thm.~\ref{amplang-normalform})
};

\node[box, anchor=north west] (evc) at (TR) {%
\textbf{$\smplang{\Sigma}\equiv \smplang{\step}$} for $\Sigma$ eventually constant and rational coefficients (Thm.~\ref{thm:eventually-constant-activation-functions})
};
\node[box, anchor=north west, below=3mm of evc] (relu) {%
\textbf{$\smplang{\relu}\ \supsetneq\ \smplang{\Sigma}$}
for $\Sigma$ eventually constant and rational coefficients
(Thm.~\ref{thm:relusep})
};
\node[box, anchor=north west, below=3mm of relu] (relub) {%
\textbf{$\smplang{\relu}\ \supsetneq\ \smplang{\step}$}
(Prop.~\ref{reluvsstep})
};

\node[idealbox, anchor=north west] (stepB) at (BL) {%
\textbf{$\smplang{\step}_{\mathbb{B}}$} is
closed under Boolean operations; ML connectives and thresholds (Sec.~\ref{sec:bounded})
};

\node[idealbox, anchor=north west] (Aq) at (BR) {%
\textbf{$\amplang_{\mathbb{B}}$} with rational coefficients
is closed under $\neg$, but not under $\land$ nor Boolean $\mlDiamond$ (Prop.~\ref{amplang-negation-closure}, Prop.~\ref{prop:nonclosure})
};
\node[idealbox, anchor=north west, below=3mm of Aq] (Ar) {%
\textbf{$\amplang_{\mathbb{B}}$} with real coefficients is
not closed under $\neg$ (Prop.~\ref{prop:amplang-not-cloused-under-negation})
};
\node[idealbox, anchor=north west, below=3mm of Ar] (comp) {%
\textbf{$\amplang_{\mathbb{B}}$ and ML} incomparable on coloured graphs
(Remark~\ref{remark:incompML})\\
\textbf{$\mathcal{L}\text{-MP}^2 \subsetneq \smplang{\step}$} 
(Prop.~\ref{prop:compare.presburger})
};
\end{tikzpicture}
\label{fig:results-map}
\end{table}
}

{ 
A summary of our results and how they align with our assumptions on the types of
embeddings and queries is in \Cref{fig:results-map}.
Ideal placements (most restrictive for negative results,
least restrictive for positive results), are slightly gray-shaded. }
\colorlet{boolshade}{gray!10}

Small design choices in
message-passing---especially activation functions---can significantly affect
expressive behaviour of GNNs.  We have seen that \mplang\ provides an attractive and transparent logical framework to investigate this expressivity, not only for boolean queries, but also for numerical queries.
Our work confirms that expressive power in GNN-style models
is best viewed through the interaction of basic numerical operations and local
graph structure, rather than architectural templates alone. A more systematic
account of this perspective may help clarify which
features
truly matter for graph-based learning.

{
From a practical perspective, our results suggest that two architectural choices matter most: the activation class and whether identity layers are available. We show that for all uneven eventually constant activations, expressivity collapses to a common class, so combining several such activations does not by itself increase expressivity. By contrast, adding identity layers can increase Boolean expressivity, and unbounded activations (such as ReLU) are strictly stronger for numerical queries. In this sense, our results provide concrete and actionable information for a critical aspect of GNN architecture design.
}

Several natural questions remain open.  Theorem~\ref{thm:relu-vs-trrelu} gives a numerical query
on coloured graphs expressible in $\smplang{\relu}$ but not in $\smplang{\step}$. Does there exist a boolean such query?  
Also, while some of our positive results are valid over $d$-embeddings in general, Theorem~\ref{thm:eventually-constant-activation-functions} is stated over coloured graphs only.  Does there similarly exist a ``universal'' activation function for the uneven eventually constant activation functions, on arbitrary $d$-embeddings?

There are also several directions for further research.  Of course, the existing body of work on expressiveness of machine learning models is huge, so we focus on directions that directly link to our work presented here. For example, how does the expressive power of \mplang with \relu compare to that with other unbounded activation functions?  Also, even for bounded activation functions (cf.~Remark~\ref{limitremark}), does expressiveness formally shift when considering ``soft'' activation functions such as sigmoid or tanh?

Finally, we believe numerical logic frameworks for graph embeddings, and methods for them such as those developed in this paper, may be a good starting point to understand expressiveness issues underlying the recent interest in
algorithm learning (e.g., \cite{velickovic-cell,nerem,Wittig2026}).

\begin{acks}
Barceló is funded by ANID - Millennium Science Initiative Program - Code ICN17002, and also by the National Center for Artificial Intelligence CENIA FB210017, Basal ANID.
Geerts is supported by FWO project G019222N. Pakhomenko and Van den Bussche are
supported by the \grantsponsor{FAIR}{Flanders AI Program (FAIR)}{https://www.flandersairesearch.be/en}. Lanzinger is supported by the Vienna Science and Technology Fund (WWTF) [10.47379/ICT2201, 10.47379/VRG18013].
\end{acks}
\bibliography{refs} %

\appendix
\section{Deferred Proofs for Section 4}\label{app:amplang}

\amplangnormalform*

\begin{proof}
    We prove it with structural induction.
    \begin{itemize}
            \item[(1)] In the case that $e=1$,
            set $c_0 = 1$ and all the $c_{0,j}$ to 0.
            \item[($P_k$)] In the case that $e = P_k$ for some $k$,
            set $c_{0,k} = 1$ and the $c_{0,j}$ for $j \neq k$ to 0.
            Set $c_0$ to 0.
            \item[(+)] For $e=f+g$,
            we have by induction hypothesis that for all graphs $G$
            and $d$-embeddings $\gamma$ on them,
            $f(G,\gamma)(v)=\sum_{i=0}^{n_1}
            \big[ c^f_i\mpDiamond^i1 + \sum_{j=1}^{d} c^f_{i,j}\mpDiamond^iP_j \big]$
            and
            $g(G,\gamma)(v)=\sum_{i=0}^{n_2}
            \big[ c^g_i\mpDiamond^i1 + \sum_{j=1}^{d} c^g_{i,j}\mpDiamond^iP_j \big]$
            for some $n_1,n_2 \leq n$ and constants $c^f_i, c^f_{i,j}, c^g_i$
            and $c^g_{i,j}$,
            where $n \coloneqq \text{max}\{n_1,n_2\}$.
            Without loss of generality say $n_1 = n$.
            Define $c^e_i = c^f_i + c^g_i$
            and $c^e_{i,j} = c^f_{i,j} + c^g_{i,j}$ for $i \in [0,n_2]$ and $d \in [1,d]$,
            then $c^e_i = c^f_i$ and $c^e_{i,j} = c^f_{i,j}$
            for $i \in [n_2+1,n]$ and $d \in [1,d]$ and the result follows.
            \item[($a$)] For $e=ae'$, we have by induction hypothesis that
            ${e'}(G)(v)=\sum_{i=0}^n
            \big[ c^{e'}_i\mpDiamond^i1 + \sum_{j=1}^{d} c^{e'}_{i,j}\mpDiamond^iP_j \big]$
            for some $c^{e'}_i$ and $c^{e'}_{i,j}$, so we just set $c_i^e = ac^{e'}_i$
            and $c_{i,j}^e = ac^{e'}_{i,j}$
            and we get the result
            (and if $ae'$ is rational, $e$ is rational).
            \item[($\mpDiamond$)] In the case of $e = \mpDiamond e'$,
            we have for all graphs $G$ and all $d$-embeddings $\gamma$
            on them that
            ${e'}(G)(v)=\sum_{i=0}^n
            \big[ c^{e'}_i\mpDiamond^i1 + \sum_{j=1}^{d} c^{e'}_{i,j}\mpDiamond^iP_j \big]$
            for some $c^{e'}_i$ and $c^{e'}_{i,j}$.
            We will define our new constants for the expression $e$ based on this.
            Say $v$ has $k$ neighbours in total,
            $u_1,\ldots,u_k$.
            Note that $ (\mpDiamond^i1)(G,\gamma)(v) =
            \sum_{j=1}^{k} (\mpDiamond^{i-1}1)(G,\gamma)(u_j)$,
            and for all $j \in [d]$
            we have $(\mpDiamond^iP_j)(G,\gamma)(v) =
            \sum_{m=1}^{k} (\mpDiamond^{i-1}P_j)(G,\gamma)(u_m)$.
            We now observe the following:
            \allowdisplaybreaks
            \begin{align*}
                e(G,\gamma)(v) &= \sum_{\ell=1}^k {e'}(G,\gamma)({u_\ell}) \\
                &= \sum_{\ell=1}^k \sum_{i=0}^{n-1} \big[ c_i(\mpDiamond^i1)(G,\gamma)(u_\ell)
                + \sum_{j=1}^{d} c_{i,j}(\mpDiamond^iP_j)(G,\gamma)(u_\ell) \big] \\
                &= \sum_{i=0}^{n-1} \sum_{\ell=1}^k \big[ c_i(\mpDiamond^i1)(G,\gamma)(u_\ell)
                + \sum_{j=1}^{d} c_{i,j}(\mpDiamond^iP_j)(G,\gamma)(u_\ell) \big] \\
                &= \sum_{i=0}^{n-1} \big[ \sum_{\ell=1}^k c_i(\mpDiamond^i1)(G,\gamma)(u_\ell) 
                + \sum_{\ell=1}^k \sum_{j=1}^{d} c_{i,j}(\mpDiamond^iP_j)(G,\gamma)(u_\ell) \big]\\
                &= \sum_{i=0}^{n-1} \big[ c_i\sum_{\ell=1}^k (\mpDiamond^i1)(G,\gamma)(u_\ell) 
                + \sum_{j=1}^{d}c_{i,j}\sum_{\ell=1}^k (\mpDiamond^iP_j)(G,\gamma)(u_\ell) \big]\\
                &= \sum_{i=0}^{n-1} \big[ c_i (\mpDiamond^{i+1}1)(G,\gamma)(v) + 
                 \sum_{j=1}^{d}c_{i,j}(\mpDiamond^{i+1}P_j)(G,\gamma)(v) \big]\\
                &= \sum_{i=1}^{n} \big[ c_{i-1} (\mpDiamond^{i}1)(G,\gamma)(v) + 
                 \sum_{j=1}^{d}c_{i-1,j}(\mpDiamond^{i}P_j)(G,\gamma)(v) \big]
            \end{align*}
            So, we get our result for the expression $e$ by setting $c^e_0 = 0$,
            $c^e_i = c^{e'}_{i-1}$, and
            $c^e_{i,j} = c^{e'}_{i-1,j}$ for $i \in [1,n]$.
        \end{itemize}
        Note that in each case, if $e$ is rational,
        the equivalent expression in our normal form is also rational.
\end{proof}

\amplangnonclosure*
\begin{proof}
    The non-closure under the $\mpDiamond$
    was shown in the section.
    For the non-closure under conjunction,
    consider the following trees:
\begin{center}
\includegraphics[height=1.2cm]{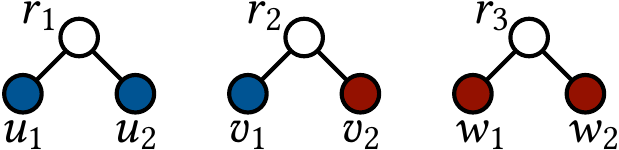}
\end{center}

The following can be shown for every \amplang expression $e$ by structural
induction: 
\begin{enumerate} 
\item $e(r_1) - e(r_2) = e(r_2) - e(r_3)$. 
\item $[e(u_1) + e(u_2)] - [e(v_1) + e(v_2)] = [e(v_1) + e(v_2)] - [e(w_1) + e(w_2)]$.
\end{enumerate}  
    The colouring is a 3-embedding $\gamma$
where $\gamma(v) = (1,0,0)$
if $v$ is white, $\gamma(v) = (0,1,0)$ if $v$ is blue,
and $\gamma(v) = (0,0,1)$ if it is red.
Then, the structural induction proof is as follows.
\begin{itemize}
    \item[(1)] For $e=1$, we have
        \begin{enumerate}
            \item $1 - 1 = 1 - 1$
            \item $[1 + 1] - [1 + 1] = [1 + 1] - [1 + 1]$
        \end{enumerate}
    \item[($P_i$)] For $e = P_i$ further make the case distinction on $i$:
        \begin{itemize}
            \item[$e = P_1$] 
                \begin{enumerate}
                    \item $1 - 1 = 1 - 1$
                    \item $[0 + 0] - [0 + 0] = [0 + 0] - [0 + 0]$
                \end{enumerate}
            \item[$e = P_2$]
                \begin{enumerate}
                    \item $0 - 0 = 0 - 0$
                    \item $[1 + 1] - [1 + 0] = [1 + 0] - [0 + 0]$
                \end{enumerate}
            \item[$e = P_3$]
                \begin{enumerate}
                    \item $0 - 0 = 0 - 0$
                    \item $[0 + 0] - [0 + 1] = [0 + 1] - [1 + 1]$
                \end{enumerate}
        \end{itemize}
    \item[($a$)] If $e = ae'$, then we have by induction and simple arithmetic that
        \begin{enumerate} 
        \item $a[e'(r_1) - e'(r_2)] = a[e'(r_2) - e'(r_3)]$. 
        \item $a \bigg[[e'(u_1) + e'(u_2)] - [e'(v_1) + e'(v_2)]\bigg] = 
        a\bigg[[e'(v_1) + e'(v_2)] - [e'(w_1) + e'(w_2)]\bigg]$
        \end{enumerate}
        which of course gives us
        \begin{enumerate} 
        \item $ae'(r_1) - ae'(r_2) = ae'(r_2) - ae'(r_3)$. 
        \item $a \bigg[[e'(u_1) + e'(u_2)] - [e'(v_1) + e'(v_2)]\bigg] = 
        [ae'(v_1) + ae'(v_2)] - [ae'(w_1) + ae'(w_2)]$.
        \end{enumerate}
        which proves the statement.
    \item[(+)] Similarly trivial.
    \item[($\mpDiamond$)] Say $e = \mpDiamond e'$.
        We look at the two cases separately.
        \begin{enumerate}
            \item We know that by induction hypothesis, 
            for $e'$ we have (2).
            Then (1) follows because $\mpDiamond e'(r_i) = e'(f(i)_1) + e'(f(i)_2)$
            where $f(1) = u$, $f(2) = v$ and $f(3) = w$.
            \item We have (1) for $e'$ by induction hypothesis,
            therefore by simple arithmetic
            \[ 2e'(r_1) - 2e'(r_2) = 2e'(r_2) - 2e'(r_3). \]
            Furthermore, $\mpDiamond e'(f(i)_j) = e'(r_i)$ for $j \in \{1,2\}$
            where $f$ as above. This proves the statement.
        \end{enumerate}
\end{itemize}
\end{proof}

In preparation for the proof of \Cref{prop:amplang-not-cloused-under-negation} we prepare two lemmas. The first is a well known standard fact about polynomials and we recall it simply for ease of presentation.

\begin{lemma}\label{lem:eventual-sign}
Let $p\in\mathbb{R}[x]$ be non-zero and let $d:=\deg p$. Then there exist $N\in\mathbb{N}$ and $c>0$ such that
$|p(n)|\ge c n^{d}$ for all $n\ge N$. In particular, $p(n)$ has eventually constant sign, and if $d\ge 1$ then
$|p(n)|\to\infty$ as $n\to\infty$.
\end{lemma}
\begin{proof}
Let $p(x)=c_d x^{d}+\cdots$ with $c_d\neq 0$ for $d \ge 1$. Then $p(n)/(c_d n^{d})\to 1$ as $n\to\infty$, so for all large $n$,
$|p(n)|\ge \frac{|c_d|}{2}n^{d}$.
\end{proof}

The following is based on a standard construction of Liouville numbers. Our proof of \Cref{prop:amplang-not-cloused-under-negation} uses a slightly specialised form of the Liouville property and we explicitly show below how it can be observed in a standard construction of Liouville numbers. In particular we add sign alternation to the original construction by Liouville \cite{liouville1851classes}, a known alternative approach to construction via continued fractions (see, e.g., \cite{bugeaud2004approximation}).

\begin{lemma}\label{lem:two-sided-liouville-n}
There exists an irrational $r>0$ and constants $c_M>0$ (for each $M\in\mathbb{N}$) such that for every $M$
there are infinitely many pairs $(a,b)\in\mathbb{N}^2$ with
$0<ra-b<c_M\,(a+b)^{-M}$ and infinitely many pairs $(a,b)\in\mathbb{N}^2$ with
$0<b-ra<c_M\,(a+b)^{-M}$.
\end{lemma}
\begin{proof}
    Let
\[
r := 1+\sum_{t\ge 1}(-1)^t 10^{-t!}.
\]
For $m\ge 1$, define the \emph{$m$-th truncation} (partial sum)
\[
r_m := 1+\sum_{t=1}^{m}(-1)^t 10^{-t!}.
\]
Set $q_m:=10^{m!}$ and $p_m:=q_m r_m\in\mathbb{Z}$, so $r_m=p_m/q_m$. 

Let $\delta_m:=r-r_m$ and let $(a_m,b_m)=(q_m,p_m)\in\mathbb{N}^2$ and $\varepsilon_m:=ra_m-b_m=q_m\delta_m$.
Since the series for $r$ is alternating with strictly decreasing term magnitudes, $\delta_m$ has sign $(-1)^{m+1}$ and 
$|\delta_m| < 2\cdot 10^{-(m+1)!}$.
Therefore $\varepsilon_m$ alternates in sign (hence is $>0$ for infinitely many $m$ and $<0$ for infinitely many $m$) and
\[
0<|\varepsilon_m|
= q_m|\delta_m|
< 2\cdot 10^{m!-(m+1)!}
= \frac{2}{10^{m\cdot m!}}
= \frac{2}{q_m^{m}}.
\tag{$\ast$}
\]

Note that $\sum_{t\ge 1}10^{-t!}\le \sum_{k\ge 1}10^{-k}=1/9$, and therefore $r_m\in[8/9,10/9]$ for all $m$ and in particular $p_m>0$.
Let $n_m:=a_m+b_m=q_m+p_m=q_m(1+r_m)$. Since $r_m\le 10/9$, we have $n_m\le (19/9)q_m$, i.e. $q_m\ge (9/19)n_m$.
Fix $M\in\mathbb{N}$. For all $m\ge M$, $(\ast)$ gives
$0<|\varepsilon_m|<2/q_m^m\le 2/q_m^M \le 2(19/9)^M\,n_m^{-M}$.
So with $c_M=2(19/9)^M$, the claim follows by taking infinitely many $m$ with $\varepsilon_m>0$ and infinitely many with $\varepsilon_m<0$.
\end{proof}

\amplangnotclosedundernegation*
\begin{proof}
        We show that with even a single real coefficient, \amplang is not closed under negation. Consider a family of graphs $\{G(a,b)\}_{a,b\in\mathbb{N}_0}$, where $G(a,b)$ is a star graph with central node $v$, that has $a$ red neighbours and $b$ blue neighbours. Consider the expression $e=r \cdot  \mpDiamond \mathsf{red}-\mpDiamond \mathsf{blue}$, where $r>0$ is an irrational number satisfying \Cref{lem:two-sided-liouville-n}.  We write $f(a,b)=a\cdot r - b$ for the corresponding function that matches the numeric evaluation of $e$ at $v$ on $G(a,b)$.

Suppose towards a contradiction that there is an \amplang expression $e'$ that expresses the negation of $e$ at $v$ (i.e., $e(v) > 0 \iff e'(v) \leq 0)$. 
By the star shape of the graph, $\mpDiamond^i \mathsf{red}$ equals $0$ for even $i$, and equals $a (a+b)^k$ for odd $i=2k+1$, and analogously for $\mathsf{blue}$. Then by \Cref{amplang-normalform},
$e'$ must be of form $\sum_i \alpha_i \cdot \mpDiamond^{2i+1} \mathsf{red} + \sum_j \beta_j \cdot \mpDiamond^{2j+1} \mathsf{blue}+\gamma$. 
Numerically, this corresponds to the function $g(a,b) = \sum_k^K( \alpha_k a n^k +  \beta_k b n^k )+ \gamma$, where $n=a+b$.

Our goal is therefore to show that there exists no $g$ of this form, such that $f(a,b) > 0 \iff g(a,b) \le 0$. We will focus on this setting and show that no such $g$ exists. As this covers all numerical forms that $e'$ can take (at $v$), this then also implies that $e$ has no expression in $\amplang_{\mathbb{B}}$ that expresses the negation of $e$.

Let us write $g(a,b)$ as $aA(n)+bB(n) + \gamma$. That is, $A(n) = \sum_k^K \alpha_k n^k$ and $B(n) = \sum_k^K \beta_k n^k$.
Observe $b=ra-\varepsilon$ where $\varepsilon = f(a,b)$. Then 
\[
g(a,b)=g(a,ra-\varepsilon) = a(A(n)+rB(n))-\varepsilon B(n) + \gamma
\]
Let $C(n) := A(n)+rB(n)$. We first show that assuming $f(a,b)>0 \iff g(a,b)\leq 0$, then $C\equiv 0$. 

Suppose towards a contradiction that $C \not\equiv 0$. Let $d=\deg C$, by \Cref{lem:eventual-sign} there exist $N\in\mathbb{N}$ and $c>0$ such that either
$C(n)\ge c n^{d}$ for all $n\ge N$, or $C(n)\le -c n^{d}$ for all $n\ge N$.

Let $K =\deg B$ and pick $M>K+2$. By \Cref{lem:two-sided-liouville-n} we can choose infinitely many
pairs $(a,b)$ with $n=a+b$ arbitrarily large and $0<\varepsilon<c_M n^{-M}$, and also infinitely many with
$\varepsilon<0$ and $|\varepsilon|<c_M n^{-M}$.
Since $|B(n)|=O(n^K)$, along either such sequence we have $|\varepsilon| \,  |B(n)|\to 0$ as $n\to\infty$.

We proceed with a case distinction on the two alternative behaviours for $C$ for large $n$.
If $C(n)\ge c n^{d}$ for all $n\ge N$, choose $M>\deg B+2$ and take $(a,b)\in\mathbb{N}^2$ from \Cref{lem:two-sided-liouville-n} with
$\varepsilon=ra-b>0$ and $n=a+b\ge N$ arbitrarily large.
Then $f(a,b)>0$ implies $g(a,b)\le 0$, and thus
$aC(n)+\gamma=g(a,b)+\varepsilon B(n)\le \varepsilon B(n)\le |\varepsilon|\,|B(n)|$.
Since $|B(n)|=O(n^{\deg B})$ and $|\varepsilon|<c_M n^{-M}$, the right-hand side tends to $0$ as $n\to\infty$.
Moreover, $n=(1+r)a-\varepsilon$, hence $a=(n+\varepsilon)/(1+r)$. For $n$ large we have $|\varepsilon|<1$, so $n+\varepsilon\ge n-1$ and
$a\ge (n-1)/(1+r)$. Therefore, for $n\ge N$,
$aC(n)+\gamma \ge \frac{n-1}{1+r}\cdot c n^{d}+\gamma \to +\infty$,
contradicting $aC(n)+\gamma\le |\varepsilon|\,|B(n)|\to 0$.

If $C(n)\le -c n^{d}$ for all $n\ge N$, take $(a,b)\in\mathbb{N}^2$ from \Cref{lem:two-sided-liouville-n} with $\varepsilon<0$ and $n\ge N$ arbitrarily large.
Then $f(a,b)\le 0$ implies $g(a,b)>0$, and thus
$aC(n)+\gamma=g(a,b)+\varepsilon B(n)>\varepsilon B(n)\ge -|\varepsilon|\,|B(n)|$.
For $n$ large we have $|\varepsilon|<1$ and thus $a=(n+\varepsilon)/(1+r)\ge (n-1)/(1+r)$, hence
$aC(n)+\gamma \le -\frac{n-1}{1+r}\cdot c n^{d}+\gamma \to -\infty$,
contradicting $aC(n)+\gamma> -|\varepsilon|\,|B(n)|\to 0$.
Therefore $C\equiv 0$, i.e. $A(n)=-rB(n)$.

Since we just showed that under assumption that $e'$ expresses the boolean negation of $e$, then $C \equiv 0$, we also have the following:
\begin{equation}
    g(a,b) = -(ra-b)B(a+b)+\gamma.
\end{equation}
We show that this implies $\gamma\le 0$.
Assume towards a contradiction that $\gamma>0$.
If $B\equiv 0$, then $g(a,b)\equiv \gamma$ and choosing e.g.\ $(a,b)=(1,0)$ (so $f(1,0)=r>0$) forces $\gamma\le 0$.
Hence assume $B\not\equiv 0$. By \Cref{lem:eventual-sign} there exists $N\in\mathbb{N}$ such that either
$B(n)>0$ for all $n\ge N$ or $B(n)<0$ for all $n\ge N$.

Let $K=\deg B$ and choose $M>K+1$. By \Cref{lem:two-sided-liouville-n} there are infinitely many pairs $(a,b)\in\mathbb{N}^2$
with $\varepsilon=ra-b>0$ and $\varepsilon<c_M(a+b)^{-M}$; in particular, we may take such pairs with $n=a+b\ge N$.
For each such pair we have $f(a,b)>0$ and thus $g(a,b)\le 0$, i.e.
$\gamma \le \varepsilon\,B(n)$.

If $B(n)<0$ for all $n\ge N$, then $\gamma \le \varepsilon B(n) < 0$, hence $\gamma\le 0$.
Otherwise $B(n)>0$ for all $n\ge N$. Since $B$ is a polynomial of degree $K$, there is a constant $D>0$ such that
$B(n)\le D n^K$ for all $n\ge N$. Therefore, for the pairs above,
$\gamma \le \varepsilon B(n) \le c_M \,n^{-M}\, D \,n^K = c_MD\,n^{K-M}$.
As $K-M<0$ and $n\to\infty$ along infinitely many  pairs $a,b \in \mathbb{N}^2$, the right-hand side tends to $0$, contradicting $\gamma>0$.
Thus in all cases $\gamma\le 0$.

However, $f(0,0)\le 0$ and thus by assumption $g(0,0) = \gamma >0$, i.e.\ $\gamma>0$, a contradiction.
\end{proof}
We note that our proof requires $r$ to be a Liouville number, a rather particular type of irrational number. In particular, this means $r$ is not algebraic. It remains an interesting open question  whether $\amplang_\mathbb{B}$ with algebraic irrational coefficients is closed under negation.

\section{Deferred Proof for Section 5}\label{app:bounded}

\eventuallyconstantfunctions*

\begin{proof}
    Let $e$ be a rational $\smplang{f}$-expression,
    where $f$ is an eventually constant function
    with endpoints $t^-$, $t^+$
    such that $f(t^+) \neq f(t^-)$.
    Let $f^- \coloneqq f(t^-)$
    and $f^+ \coloneqq f(t^+)$.
    We show by induction on the structure of $e$
    that there is a numerically equivalent expression $e'$ in $\smplang{\step}$.
    \begin{itemize}
        \item If $e = 1$ or $e = P_i$ it is already a $\smplang{\step}$-expression.
        \item For the cases $e = e_1 + e_2$, $e = ae_1$ or $e = \mpDiamond e_1$
        it follows by the induction hypothesis.
        \item If $e = f(e_1)$.
    By induction, we can assume $e_1$
    to be a $\smplang{\step}$-expression.
    Let $d$ be the lowest common denominator
    of the coefficients in $e_1$.
    Note that on any pointed $k$-coloured graph $(G,c,v)$,
    we have that since $e_1$ is a rational
    $\smplang{\step}$-expression
    and $d$ is the common denominator
    of its coefficients,
    $e_1(G,c)(v)$ only takes on values
    of the form $\nicefrac{\ell}{d}$ for some $\ell \in \Z$.
    Now note that there are $\ell^-,\ell^+ \in \Z$ such that:
    \begin{align*}
        \forall \ell \in \Z, \ell \leq \ell^-: f(\nicefrac{\ell}{d}) &= f^- \\
        \forall \ell \in \Z, \ell \geq \ell^+: f(\nicefrac{\ell}{d}) &= f^+.
    \end{align*}
    As a result, $f(e_1)$ only takes on values $f(x)$
    where
    $x \in D \coloneqq [\nicefrac{\ell^-}{d}, \nicefrac{\ell^-+1}{d}, \ldots, \nicefrac{\ell^+}{d}]$.
    Thus, in order to give a $\smplang{\step}$-expression
    that is numerically equivalent to $f(e_1)$,
    we have to build a function out of $\step$ using affine combinations
    that coincides with $f$ on all of the values $D$,
    which are finitely many.
    
    We start by defining the alternative step function
    $\step'(x) \coloneqq 1-\step(-x)$. Note that this only changes in the case
    where $x$ is 0, that is
        \[ \step'(x) = \begin{cases}
            1, & x \geq 0 \\ 0, & x < 0.
        \end{cases} \]
    Further, we define the transformed step function
    \[ t_{p_1,p_2}(x) \coloneqq p_2\step'(x - p_1). \]
    Note that
    \[ t_{p_1,p_2}(x) = \begin{cases}
        p_2, & x \geq p_1 \\ 0, & x < p_1.
    \end{cases} \]
    We also define the 'blip' function
    \[ s_{p_1,p_2} = t_{p_1,p_2}(x) - t_{(p_1 + \nicefrac{1}{d}),p_2}(x). \] Note that
    \[ s_{p_1,p_2}(x) = \begin{cases}
        0, & x \geq p_1 + \nicefrac{1}{d} \\
        p_2, & p_1 \leq x < p_1 + \nicefrac{1}{d} \\ 0, & x < p_1.
    \end{cases} \]
    Lastly, we defined the reverse transformed step function
    \[ t'_{p_1,p_2} := p_2[1-\step(x-p_1)] \]
    where
    \[ t'_{p_1,p_2}(x) = \begin{cases}
        0, & x > p_1 \\ p_2, & x \leq p_1.
    \end{cases} \]
    \noindent
    Then,
    \begin{align*}
        & t'_{\nicefrac{k^-}{d},f^-}(e_1) \\
       &{} +  \sum_{i \in D}s_{i,f(i)}(e_1) \\
       &{} +  t_{\nicefrac{k^+}{d},f^+}(e_1)
    \end{align*}
    is a $\smplang{\step}$-expression and
    is numerically equivalent to $f(e_1)$:
    It returns the same values 
    when $e_1$ takes values in $D$,
    and when $e_1$ does not take values in $D$,
    $f(e_1)$ becomes constant, also coinciding
    with the expression.
    \end{itemize}
    Note that if $e$ makes use of a set of uneven eventually constant functions,
    each one can individually be emulated by $\step$ in the way outlined above,
    and so the translation also works for $\smplang{\Sigma}$
    if $\Sigma$ is any set of uneven eventually constant functions.
    For the reverse, it remains to show that $\step$
    can be emulated by any eventually constant function.
    So let $\sigma$ be an eventually constant function with endpoints $t^-$, $t^+$.
    Let $e$ be a rational $\smplang{\step}$-expression
    and $d_e$ be the lowest common denominator
    of its coefficients.
    Analogously to the preceding argument,
    $e$ will only take values of the form $\nicefrac{\ell}{d_e}$ for some $\ell \in \N$.
    First, we move the expression along the x axis by setting $e' = e + t_-$.
    Then we set $e'' \coloneqq \sigma(d_e(t^- - t^+)e')$,
    noting that for any $\ell \in \N$:
    \[ \sigma(d_e(t^- - t^+)(\nicefrac{\ell}{d_e} + t^-)) =
    \begin{cases}
        \sigma(t^-),& \nicefrac{\ell}{d_e} \leq 0 \\
        \sigma(t^+),& \nicefrac{\ell}{d_e} > 0.
    \end{cases}\]
    Finally, we move the expression
    to the correct spot on the y axis by setting
    $e'' \coloneqq (e' - \sigma(t^-)/(\sigma(t^-)-\sigma(t^+))$.
    This is where it is crucial that $\sigma$ is in fact uneven.
    Now $e''$ is a $\smplang{\sigma}$ expression
    that equates to $\step(e)$ on any possible
    value that $e$ can take.
\end{proof}

\section{Deferred Proofs for Section 6}\label{app:unbounded}

\boolmplangonrbgraph*
\begin{proof}
    We show this inductively by structural induction on $e$.
    \begin{itemize}
        \item The case $e=1$ is a constant,
        and constants are nice and symmetric.
        \item The case $e=P_i$ depends only on the color of the node,
        so all cases are constant (thus also nice and symmetric).
            \item The case $e=ae'$ follows by induction because nice functions,
            affine functions and simple functions
            are all closed under scalar multiplication.
            \item The case $e=f+g$ follows by induction,
            again because nice functions, affine functions and simple functions
            are all closed under addition.
            \item The case $e=\step(e')$ follows because
            \begin{itemize}
                \item On the red and blue nodes,
                the expression returns
                \begin{itemize}
                    \item $\beta_r^S(r,b) \coloneqq \step(S_r(r,b))$ on the red nodes, and
                    \item $\beta_b^S(r,b) \coloneqq \step(S_b(r,b))$ on the blue nodes.
                \end{itemize}
                Both are symmetric and simple.
                They are also nice:
                $S_r(r,b) \coloneqq 1\beta_r^S(r,b)$,
                and $S_b(r,b) \coloneqq 1\beta_b^S(r,b)$
                (1 is a polynomial).
                \item On the white nodes...
                \begin{itemize}
                    \item On the $w$-type nodes,
                    the expression evaluates to
                    $0 + 0 + \beta(r,b)$
                    where $\beta(r,b) \coloneqq \step(F(r,b) + \alpha(r,b) + \beta(r,b))$,
                    which is simple.
                    This fits the required shape.
                    \item And on the $w'$-type nodes,
                    it evaluates to $0 + 0 + \beta(b,r)$.
                \end{itemize}
              \end{itemize}
            \item For the interesting case $e = \mpDiamond e'$, we make a case distinction
            on which type of node is being evaluated.
            By induction there are
            functions $S_r$, $S_b$, $F$ and $f_w$ that satisfy the statement
            for $e'$.
            \begin{itemize}
                \item For the case of the red and blue nodes,
                they have one $w$-type neighbour and one $w'$ type neighbour.
                Therefore, the resulting function
                will be
                \begin{align*}
                S(r,b)& = F(r,b) + \alpha(r,b) + \beta(r,b) + F(b,r) + \alpha(b,r) + \beta(b,r)\\
                &=\sum_{i=0}^k p_i(r,b)\beta_i(r,b)
                +\sum_{i=0}^k p_i(b,r)\beta_i(b,r)
                +\alpha(r,b)+\beta(r,b)+\alpha(b,r)+\beta(b,r)\\
                &=\sum_{i=0}^k \underbrace{(p_i(r,b)+p_i(b,r))}_{p_i'(r,b)}\beta_i(r,b)
                +\underbrace{(\alpha(r,b)+\alpha(b,r))}_{p'_{k+1}(r,b)}1+
                1\underbrace{(\beta(r,b)+\beta(b,r))}_{\beta_{k+2}(r,b)}
                \end{align*}
                which is symmetric and nice.
                \item For $w$-type nodes,
                the expression returns $rS_r(r,b) + bS_b(r,b)$.
                Now, we define the polynomial $p(r,b) \coloneqq r$.
                Then at $w$, the expression returns
                $F(r,b) \coloneqq p(r,b)S_r(r,b) + p(b,r)S_b(r,b)$
                which is again nice (nice functions
                are closed under multiplication with a polynomial,
                and addition).
                \item For $w'$-type nodes, the result is analogous
                but with the number of red and blue neighbours exchanged:
                The expression evaluates to
                $bS_r(r,b) + rS_b(r,b) = p(b,r)S_r(r,b) + p(r,b)S_b(r,b)$.
                By symmetry of $S$, this equals $p(b,r)S_r(b,r) + p(r,b)S_b(b,r)$
                which is exactly $F(b,r)$.
            \end{itemize}
        \end{itemize}
\end{proof}

\blobinomialsvsrelu*
\begin{proof}
To prove something slightly stronger,
let the affine function in the statement become polynomial
and assume that there is a polynomial $p: \R^2 \to \R$,
nice $F$ and simple $\beta$
such that
\[
G(x,y) \coloneqq p(x,y) + \sum_{j=1}^k q_j(x,y)\beta_j(x,y)  + \beta(x,y)
\]
satisfies $G(x,y)=|x-y|$ for all $(x,y)\in\N^2$.

We first use that the $\beta_j$ are simple and hence take only finitely many values. This allows us to view $G$ as one of finitely many {\em base polynomials}. More specifically, for each vector $\vec{c}=(c_1,\ldots,c_k)$ with each $c_j$ in the (finite) range of $\beta_j$, we define 
\[
Q_{\vec{c}}(x,y):=p(x,y)+\sum_{j=1}^k c_j\,q_j(x,y)\in\R[x,y].
\]
In particular, if  $\beta_1(x,y)=c_1,\ldots,\beta_k(x,y)=c_k$, then we can view $G$ as the polynomial
\[
G(x,y)=Q_{(c_1,\ldots,c_k)}(x,y)+\beta(x,y).
\]

We fix $\delta\in\N$ and consider the following diagonal line $L_\delta$ in the grid $\N^2$: 
\[
L_\delta:=\bigl\{(m+\delta,m)\bigm\vert m\in\N\bigr\}.\]
Since, by assumption, $G(x,y)=|x-y|$ on $\N^2$, we have
\[
F(m+\delta,m)=\delta\qquad\text{for all } m\in\N,
\]
so $F$ is constant on each $L_\delta$. We note that $L_\delta$ contains infinitely many grid points. By contrast,
\[
B_\delta:=\bigl\{\bigl(\beta_1(m+\delta,m),\ldots,\beta_k(m+\delta,m),\beta(m+\delta,m)\bigr)\bigm\vert m\in\N\bigr\}
\]
contains only finitely many values because the $\beta_j$ and $\beta$ are simple. By the pigeonhole principle, this implies that there exists an 
element 
$(\vec{c}(\delta),b_\delta)\in B_\delta$ that is visited infinitely often along $L_\delta$.
That is, there exists an infinite set $S_\delta\subseteq\N$ such that for all $m\in S_\delta$,
\[
\bigl(\beta_1(m+\delta,m),\ldots,\beta_k(m+\delta,m),\beta(m+\delta,m)\bigr)=(\vec{c}(\delta),b_\delta),
\]
and thus for all $m\in S_\delta$,
\[
Q_{\vec{c}(\delta)}(m+\delta,m)+b_\delta=\delta.
\]
Hence, we also have, as a polynomial identity in $x$, that
\[
Q_{\vec{c}(\delta)}(x+\delta,x)\equiv \delta-b_\delta.
\]

We next use that all $\beta_j$ are symmetric. For each $j=1,\ldots,k$, $\delta\in\N$ and $m\in S_\delta$, by symmetry
\[
\beta_j(m,m+\delta)=\beta_j(m+\delta,m)=c_j(\delta),
\]
where $\vec{c}(\delta)=(c_1(\delta),\ldots,c_k(\delta))$.
Since $\beta$ has finite range, passing to an infinite subset $S_{\delta}'\subseteq S_\delta$, we may also assume that
\[
Q_{\vec{c}(\delta)}(m,m+\delta)+b_\delta'=\delta,
\]
for some constant $b_\delta'$ in the range of $\beta$,
for all $m\in S_\delta'$. Hence, we can also conclude that
\[
Q_{\vec{c}(\delta)}(x,x+\delta)\equiv \delta-b_\delta'
\]
as polynomials.

Let us now vary $\delta\in\N$ and consider
\[
C:=\bigl\{\bigl(\vec{c}(\delta),b_\delta,b_\delta'\bigr)\bigm\vert \delta\in\N\bigr\},
\]
where $\vec{c}(\delta)$, $b_\delta$ and $b_\delta'$ are as defined above. We remark again that, being values in the range of the simple functions,
$C$ only contains finitely many values. By the pigeonhole principle, there exists a $(\vec{c},b,b')\in C$ and an infinite set $D\subseteq\N$ such that for all $d\in D$, 
\[
(\vec{c}(d),b_d,b_d')=(\vec{c},b,b'),
\]
and hence we also have the following identities as polynomials
\[
Q_{\vec{c}}(x+d,x)\equiv d-b \quad\text{and}\quad Q_{\vec{c}}(x,x+d)\equiv d-b',
\]
for all $d\in D$. We let
\[
Q(x,y):=Q_{\vec{c}}(x,y)=p(x,y)+\sum_{j=1}^k c_j q_j(x,y).
\]

We now come to the key part of the proof where we observe that $Q(x,y)$ can be viewed as a polynomial in $x-y$. To this aim we make the change of variables $u=x-y$, $v=y$, and write $\widetilde Q(u,v):=Q(u+v,v)\in\R[u,v]$. For every $d\in D$, the specialization $v\mapsto\widetilde Q(d,v)=Q(v+d,v)$ is constant in $v$. Indeed, we know that $Q(v+d,v)=d-b$.
Writing
\[
\widetilde Q(u,v)=C_0(u)+C_1(u)v+\cdots+C_k(u)v^k \quad\text{with each } C_j\in\R[u],
\]
we get $C_j(d)=0$ for all $j\ge1$ and infinitely many $d$ (all $d$ in $D$). As a consequence, we get zero polynomials $C_j\equiv 0$ for $j\ge 1$. Thus $\widetilde Q(u,v)=C_0(u)$ depends only on $u$. That is, $Q(x,y)$ can be viewed as the polynomial
\[
Q(x,y)=\widetilde Q(x-y,y)=C_0(x-y).
\]

We are now ready to obtain a contradiction. Indeed, from the identities above, for all $d\in D$ we have
\[
C_0(d)=d-b,\qquad C_0(-d)=d-b'.
\]
Adding yields
\[
C_0(d)+C_0(-d)=2d-(b+b')
\]
for all $d\in D$.
The left-hand side is an even polynomial in $d$, that is, replacing $d$ by $-d$ has no effect. By contrast, the right-hand side has a non-zero odd part $2d$ and hence is not even.
Since two polynomials that agree on an infinite set (all $d\in D$) must be identical, this is impossible. The contradiction completes the proof.
\end{proof}

\end{document}